\documentclass{article}

\usepackage{times}
\usepackage{pifont}
\usepackage{bm}
\usepackage{amsmath}
\usepackage{amssymb}
\usepackage{hyperref}
\usepackage{url}
\usepackage{algorithm}
\usepackage{algorithmic}
\usepackage{xfrac}
\usepackage{graphicx}
\usepackage{tensor}
\usepackage{float}
\usepackage{wrapfig}
\usepackage{color}
\usepackage{framed}
\usepackage{eurosym}
\usepackage{supertabular}
\usepackage{colortbl}
\usepackage{sidecap}
\newenvironment{claim}[1]{\par\noindent\underline{Claim:}\space#1}{}
\newenvironment{claimproof}[1]{\par\noindent\underline{Proof:}\space#1}{\hfill $\blacksquare$}
\usepackage{times}
\usepackage{graphicx} 
\usepackage{subfigure} 
\usepackage{amsthm}
\usepackage{natbib}

\usepackage{algorithm}
\usepackage{algorithmic}

\usepackage{hyperref}

\usepackage[accepted]{icml2015}

\newcommand{\vTask}{{t}}
\newcommand{\vTaskOne}{{t_{1}}}
\newcommand{\vTaskJ}{{t_{j}}}
\newcommand{\vTaskR}{{t_{r}}}
\newcommand{\vTaskK}{{t_{k}}}
\newcommand{\vCurrentRound}{r}
\newcommand{\vNumTotalRounds}{R}
\newcommand{\vNumTotalTasks}{{|\mathcal{T}|}}

\newcommand{\vTempIndexA}{i}

\newtheorem{theorem}{Theorem}
\newtheorem{assumption}{Assumption}
\newtheorem{lemma}{Lemma}

\newtheorem*{mylemma1}{Lemma 1}
\newtheorem*{mylemma2}{Lemma 2}
\newtheorem*{mylemma3}{Lemma 3}
\newtheorem*{mylemma4}{Lemma 4}
\newtheorem*{mylemma5}{Lemma 5}
\newtheorem*{mytheo}{Theorem 1}
\icmltitlerunning{Safe Policy Search for Lifelong Reinforcement Learning with Sublinear Regret}

\begin{document}

\setlength{\abovedisplayskip}{4.0pt plus 1.0pt minus 2.0pt}
\setlength{\belowdisplayskip}{4.0pt plus 1.0pt minus 2.0pt}
\setlength{\abovedisplayshortskip}{0.0pt plus 2.0pt}
\setlength{\belowdisplayshortskip}{0.0pt plus 2.0pt minus 1.0pt} 

\setlength{\parskip}{5pt  minus 1pt} 

\twocolumn[
\icmltitle{Safe Policy Search for Lifelong Reinforcement Learning with Sublinear Regret}
\icmlauthor{Haitham Bou Ammar}{haithamb@seas.upenn.edu}
\icmlauthor{Rasul Tutunov}{tutunov@seas.upenn.edu}
\icmlauthor{Eric Eaton}{eeaton@cis.upenn.edu}
\icmladdress{University of Pennsylvania, Computer and Information Science Department, Philadelphia, PA 19104 USA}

\icmlkeywords{lifelong learning, online multi-task learning, reinforcement learning, policy gradients, lifelong policy search}

\vskip 0.20in
]

\begin{abstract}
Lifelong reinforcement learning provides a promising framework for developing versatile agents that can accumulate knowledge over a lifetime of experience and rapidly learn new tasks by building upon prior knowledge.  However, current lifelong learning methods exhibit non-vanishing regret as the amount of experience increases, and include limitations that can lead to suboptimal or unsafe control policies.  To address these issues, we develop a lifelong policy gradient learner that operates in an adversarial setting to learn multiple tasks online while enforcing safety constraints on the learned policies.  We demonstrate, for the first time, {\em sublinear regret} for lifelong policy search, and validate our algorithm on several benchmark dynamical systems and an application to quadrotor control. \vspace{-1.5em}
\end{abstract}

\section{Introduction}

Reinforcement learning (RL)~\cite{Busoniu,Sutton} often requires substantial experience before achieving acceptable performance on individual control problems. One major contributor to this issue is the \emph{tabula-rasa} assumption of typical RL methods, which learn from scratch on each new task. In these settings, learning performance is directly correlated with the quality of the acquired samples. Unfortunately, the amount of experience necessary for high-quality performance increases exponentially with the tasks' degrees of freedom, inhibiting the application of RL to high-dimensional control problems.


When data is in limited supply, transfer learning can significantly improve model performance on new tasks by reusing previous learned knowledge during training \cite{Taylor09,Lazaric1,lazaric2011transfer, ferrante2008transfer,HaithamTwo}. 
Multi-task learning (MTL) explores another notion of knowledge transfer, in which task models are trained simultaneously and share knowledge during the joint learning process~\cite{Wilson,Zhang_flexiblelatent}.


In the \emph{lifelong learning} setting~\cite{Thrun96discoveringstructure,Thrun_1996_679}, 
which can be framed
as an online MTL problem, agents acquire knowledge incrementally
by learning multiple tasks consecutively over their lifetime.  
Recently, based on the work of~\citet{Ruvolo2013ELLA} on supervised lifelong learning, \citet{BouAmmar2014Online}
developed a lifelong learner for policy gradient RL.  To ensure efficient learning over consecutive tasks, these works employ a second-order Taylor expansion around the parameters that are (locally) optimal for each task without transfer. This assumption simplifies
the MTL objective into a weighted quadratic form for online learning, but since it is based on single-task learning, this technique can lead to parameters far from globally optimal.  
Consequently, the success of these methods for RL highly depends on the policy initializations, which must lead to near-optimal trajectories for meaningful updates. Also, since their objective functions average loss over all tasks, these methods exhibit non-vanishing regrets of the form $\mathcal{O}(R)$, where $R$ is the total number of rounds in a non-adversarial setting.  

In addition, these methods may produce control policies with unsafe behavior (i.e., capable
of causing damage to the agent or environment, catastrophic failure,
etc.). This is a critical issue in robotic control, where unsafe control
policies can lead to physical damage or user injury. This problem is caused by using constraint-free optimization
over the shared knowledge during the transfer process, which may lead to uninformative or
unbounded policies.

In this paper, we address these issues by proposing the first \emph{safe lifelong learner}
for policy gradient RL operating in an adversarial framework. Our approach rapidly learns 
high-performance \emph{safe control policies} based on the agent's
previously learned knowledge and safety constraints on each task, accumulating knowledge over multiple consecutive tasks to optimize overall performance.  We theoretically analyze the regret
exhibited by our algorithm, showing \emph{sublinear} dependency of
the form $\mathcal{O}(\sqrt{\vNumTotalRounds})$ for $\vNumTotalRounds$ 
rounds, thus outperforming current methods. We then evaluate our approach empirically on a set of dynamical systems. 

\section{Background}

\label{Sec:Background}

\subsection{Reinforcement Learning}
An RL agent sequentially chooses actions to minimize its expected
cost. Such problems are formalized as Markov decision processes (MDPs)
$\left\langle \mathcal{X},\mathcal{U},\mathcal{P},\bm{c},\gamma\right\rangle $,
where $\mathcal{X}\subset\mathbb{R}^{d}$ is the (potentially infinite)
state space, $\mathcal{U}\in\mathbb{R}^{d_{a}}$ is the set of all
possible actions, $\mathcal{P}:\mathcal{X}\times\mathcal{U}\times\mathcal{X}\rightarrow[0,1]$
is a state transition probability describing the system's dynamics,
$\bm{c}:\mathcal{X}\times\mathcal{U}\times\mathcal{X}\rightarrow\mathbb{R}$
is the cost function measuring the agent's performance, and $\gamma\in[0,1]$
is a discount factor. At each time step $m$, the agent is in state
$\bm{x}_{m}\in\mathcal{X}$ and must choose an action $\bm{u}_{m}\in\mathcal{U}$,
transitioning it to a new state $\bm{x}_{m+1}\sim\mathcal{P}\left(\bm{x}_{m+1}|\bm{x}_{m},\bm{u}_{m}\right)$
and yielding a cost $\bm{c}_{m+1}=\bm{c}(\bm{x}_{m+1},\bm{u}_{m},\bm{x}_{m})$.
The sequence of state-action pairs forms a trajectory $\bm{\tau}=\left[\bm{x}_{0:M-1},\bm{u}_{0:M-1}\right]$
over a (possibly infinite) horizon $M$. A policy $\pi:\mathcal{X}\times\mathcal{U}\rightarrow[0,1]$
specifies a probability distribution over state-action pairs, where
$\pi\left(\bm{u}|\bm{x}\right)$ represents the probability of selecting
an action $\bm{u}$ in state $\bm{x}$. The goal of RL is to find
an optimal policy $\pi^{\star}$ that minimizes the total expected cost.

\textbf{Policy search methods} have shown success in solving high-dimensional
problems, such as robotic control~\cite{Kober, Peters_NN_2008,Sutton00policygradient}.
These methods represent the policy $\pi_{\bm{\alpha}}(\bm{u}|\bm{x})$
using a vector $\bm{\alpha}\in\mathbb{R}^{d}$ of control parameters.
The optimal policy $\pi^{\star}$ is found by determining the parameters
$\bm{\alpha}^{\star}$ that minimize the expected average cost: 
\begin{equation}
l(\bm{\alpha})=\sum_{k=1}^{n}p_{\bm{\alpha}}\!\left(\bm{\tau}^{(k)}\right)\bm{C}\!\left(\bm{\tau}^{(k)}\right)\enspace,\label{Eq:St_PG}
\end{equation}
where $n$ is the total number of trajectories, and $p_{\bm{\alpha}}\!\left(\bm{\tau}^{(k)}\right)$
and $\bm{C}\!\left(\bm{\tau}^{(k)}\right)$ are the probability and
cost of trajectory $\bm{\tau}^{(k)}$: 
\begin{align}
\begin{split}p_{\bm{\alpha}}\left(\bm{\tau}^{(k)}\right) & =\mathcal{P}_{0}\left(\bm{x}_{0}^{(k)}\right)\prod_{m=0}^{M-1}\mathcal{P}\left(\bm{x}_{m+1}^{(k)}|\bm{x}_{m}^{(k)},\bm{u}_{m}^{(k)}\right)\\[-0.4em]
 & \hspace{7.6em}\times\pi_{\bm{\alpha}}\left(\bm{u}_{m}^{(k)}|\bm{x}_{m}^{(k)}\right)
\end{split}
\\
\bm{C}\left(\bm{\tau}^{(k)}\right) & =\frac{1}{M}\sum_{m=0}^{M-1}\bm{c}\left(\bm{x}_{m+1}^{(k)},\bm{u}_{m}^{(k)},\bm{x}_{m}^{(k)}\right) \enspace ,
\end{align}
with an initial state distribution $\mathcal{P}_{0}:\mathcal{X}\rightarrow[0,1]$. We handle a constrained version of policy search, in which optimality not only corresponds to minimizing the total expected cost, but also to ensuring that the policy satisfies safety constraints. These constraints vary between applications, for example corresponding to maximum joint torque or prohibited physical positions.

\subsection{Online Learning \& Regret Analysis}
In this paper, we employ a special form of \emph{regret minimization games}, which we briefly review here. A regret minimization game is a triple $\langle \mathcal{K}, \mathcal{F}, R\rangle$, where $\mathcal{K}$ is a non-empty decision set, $\mathcal{F}$ is the set of moves of the adversary which contains bounded convex functions from $\mathbb{R}^{n}$ to $\mathbb{R}$, and $R$ is the total number of rounds. The game proceeds in rounds, where at each round $j=1,\dots, R$, the agent chooses a prediction $\bm{\theta}_{j} \in \mathcal{K}$ and the environment (i.e., the adversary) chooses a loss function $l_{j} \in \mathcal{F}$. At the end of the round, the loss function $l_{j}$  is revealed to the agent and the decision $\bm{\theta}_{j}$ is revealed to the environment. In this paper, we handle the full-information case, where the agent may observe the entire loss function $l_{j}$ as its feedback and can exploit this in making decisions. The goal is to minimize the cumulative regret $\sum_{j=1}^{R} l_{j}(\bm{\theta}_{j})  - \text{inf}_{\bm{u} \in \mathcal{K}} \left[\sum_{j=1}^{R} l_{j}(\bm{u})\right]$.
%
When analyzing the regret of our methods, we use a variant of this definition to handle the lifelong RL case: 
\begin{equation*}
\mathfrak{R}_{R}=\sum_{j=1}^{R}l_{t_{j}}(\bm{\theta}_{j})-\inf_{u\in \mathcal{K}}\left[\sum_{j=1}^{R}l_{t_{j}}(\bm{u})\right] \enspace ,
\end{equation*}
where $l_{t_{j}}(\cdot)$ denotes the loss of task $t$ at round $j$. 

For our framework, we adopt a variant of regret minimization called ``Follow the Regularized Leader,'' which minimizes regret in two steps. First, the unconstrained solution $\tilde{\bm{\theta}}$ is determined (see Sect.~\ref{Sec:Unconstrained}) by solving an unconstrained optimization over the accumulated losses observed so far. 
Given $\tilde{\bm{\theta}}$, the constrained solution is then determined by learning a projection into the constraint set via Bregman projections (see~\citet{AYBKSSz13}).

\section{Safe Lifelong Policy Search}

We adopt a lifelong learning framework in which the agent learns multiple
RL tasks consecutively, providing it the opportunity to transfer knowledge
between tasks to improve learning. Let $\mathcal{T}$ denote the set
of tasks, each element of which is an MDP. At any time, the learner
may face any previously seen task, and so must strive to maximize
its performance across all tasks. The goal is to learn optimal policies
$\pi_{\bm{\alpha}_{1}^{\star}}^{\star},\dots,\pi_{\bm{\alpha}_{\vNumTotalTasks}^{\star}}^{\star}$
for all tasks, where policy $\pi_{\bm{\alpha}_{t}^{\star}}^{\star}$
for task $t$ is parameterized by $\bm{\alpha}_{t}^{\star}\in\mathbb{R}^{d}$.
In addition, each task is equipped with safety constraints to ensure
acceptable policy behavior: $\bm{A}_{\vTask}\bm{\alpha}_{\vTask}\leq\bm{b}_{\vTask}$,
with $\bm{A}_{\vTask}\in\mathbb{R}^{d\times d}$ and $\bm{b}_{\vTask}\in\mathbb{R}^{d}$
representing the allowed policy combinations. The precise form of
these constraints depends on the application domain, but this formulation supports constraints on (e.g.) joint torque, acceleration, position,
etc.

At each round $j$, the learner observes a set of $n_{\vTaskJ}$
trajectories $\left\{ \bm{\tau}_{\vTaskJ}^{(1)},\dots,\bm{\tau}_{\vTaskJ}^{(n_{\vTaskJ})}\right\} $
from a task $\vTaskJ\in\mathcal{T}$, where each trajectory has
length $M_{\vTaskJ}$. To support knowledge transfer between tasks,
we assume that each task's policy parameters $\bm{\alpha}_{\vTaskJ}\in\mathbb{R}^{d}$
at round $j$ can be written as a linear combination of a shared latent
basis $\bm{L}\in\mathbb{R}^{d\times k}$ with coefficient vectors
$\bm{s}_{\vTaskJ}\in\mathbb{R}^{k}$; therefore, $\bm{\alpha}_{\vTaskJ}=\bm{L}\bm{s}_{\vTaskJ}$.
Each column of $\bm{L}$ represents a chunk of transferrable knowledge;
this task construction has been used successfully in previous multi-task
learning work \cite{daume12gomtl,Ruvolo2013ELLA,BouAmmar2014Online}.
Extending this previous work, we ensure that the shared knowledge
repository is ``informative'' by incorporating bounding constraints
on the Frobenius norm $\|\cdot\|_{\mathsf{F}}$ of $\bm{L}$. Consequently, the optimization
problem after observing $\vCurrentRound$ rounds is: 
\begin{align}
\min_{\bm{L},\bm{S}} & \sum_{j=1}^{\vCurrentRound}\left[\eta_{\vTaskJ}l_{\vTaskJ}\left(\bm{L}\bm{s}_{\vTaskJ}\right)\right]+\mu_{1}\left|\left|\bm{S}\right|\right|_{\mathsf{F}}^{2}+\mu_{2}\left|\left|\bm{L}\right|\right|_{\mathsf{F}}^{2}\label{Eq:OriginalOpti}\\
 & \ \ \ \text{s.t.}\ \ \ \ \bm{A}_{\vTaskJ}\bm{\alpha}_{\vTaskJ}\leq\bm{b}_{\vTaskJ}\ \ \forall\vTaskJ\in\mathcal{I}_{\vCurrentRound}\nonumber \\
 & \ \ \ \ \ \ \ \ \ \ \ \bm{\lambda}_{\text{min}}\left(\bm{L}\bm{L}^{\mathsf{T}}\right)\geq p\ \ \text{and}\ \ \bm{\lambda}_{\text{max}}\left(\bm{L}\bm{L}^{\mathsf{T}}\right)\leq q\enspace,\nonumber 
\end{align}
where $p$ and $q$ are the constraints on $\|\bm{L}\|_{\mathsf{F}}$, $\eta_{\vTaskJ}\in\mathbb{R}$
are design weighting parameters\footnote{We describe later how to set the $\eta$'s later in Sect.~\ref{sect:MainResults} to obtain regret bounds, and leave them as variables now for generality.}, $\mathcal{I}_{\vCurrentRound}=\left\{ t_{1},\dots,t_{\vCurrentRound}\right\} $ denotes
the set of all tasks observed so far through round $\vCurrentRound$, and $\bm{S}$ is the collection of all coefficients  
\begin{align*}
\begin{split}\bm{S}(:,h)=\left\{ \begin{array}{lr}
\bm{s}_{t_{h}} & \text{if \ensuremath{t_{h}\in\mathcal{I}_{\vCurrentRound}}}\\
0 & \text{otherwise}
\end{array}\right.\end{split}
\mbox{~~~\ensuremath{\forall h\in\{1,\dots,\vNumTotalTasks\}}}\enspace.
\end{align*}

The loss function $l_{\vTaskJ}(\bm{\alpha}_{\vTaskJ})$ in Eq.~\eqref{Eq:OriginalOpti}
corresponds to a policy gradient learner for task $\vTaskJ$, as
defined in Eq.~\eqref{Eq:St_PG}. Typical policy gradient methods~\cite{Kober,Sutton00policygradient}
maximize a lower bound of the expected cost $l_{\vTaskJ}\!\left(\bm{\alpha}_{\vTaskJ}\right)$,
which can be derived by taking the logarithm and applying Jensen's
inequality:
\begin{align}
 & \log\!\left[l_{\vTaskJ}\!\left(\bm{\alpha}_{\vTaskJ}\right)\right]=\log\!\left[\sum_{k=1}^{n_{\vTaskJ}}p_{\bm{\alpha}_{\vTaskJ}}^{\left(\vTaskJ\right)}\!\left(\bm{\tau}_{\vTaskJ}^{(k)}\right)\bm{C}^{\left(\vTaskJ\right)}\!\left(\bm{\tau}_{\vTaskJ}^{(k)}\right)\right]\label{Eq:Loss}\hspace{-1em}\\[-.25em]
\nonumber 
& \geq \log\!\left[n_{\vTaskJ}\right]\! +\mathbb{E}\!\!\left[ \sum_{m=0}^{M_{\vTaskJ}-1}\!\!\!\log\!\left[\pi_{\bm{\alpha}_{\vTaskJ}}\!\!\left(\bm{u}_{m}^{\left(k,\vTaskJ\right)} \mid \bm{x}_{m}^{\left(k,\vTaskJ\right)}\!\right)\right]\! \right]_{k=1}^{n_{\vTaskJ}}\!\!\!\!\!\!\!\!\!\!\!\!+\! \text{\footnotesize const}~.\nonumber
\end{align}
Therefore, our goal is to minimize the following objective: 
\begin{align}
 & \bm{e}_{\vCurrentRound}=\sum_{j=1}^{\vCurrentRound}\!\left(\!\!-\frac{\eta_{\vTaskJ}}{n_{\vTaskJ}}\sum_{k=1}^{n_{\vTaskJ}}\sum_{m=0}^{M_{\vTaskJ}-1}\!\!\!\log\!\left[\pi_{\bm{\alpha}_{\vTaskJ}}\!\!\left(\bm{u}_{m}^{\left(k,\vTaskJ\right)}\mid\bm{x}_{m}^{\left(k,\vTaskJ\right)}\right)\right]\!\!\right)\label{Eq:RelaxedOne}\\[-0.8em]
 & \hspace{4em}+\mu_{1}\left\Vert \bm{S}\right\Vert _{\mathsf{F}}^{2}+\mu_{2}\left\Vert \bm{L}\right\Vert _{\mathsf{F}}^{2}\nonumber \\
 & \ \ \ \text{s.t.}\ \ \ \ \bm{A}_{\vTaskJ}\bm{\alpha}_{\vTaskJ}\leq\bm{b}_{\vTaskJ}\ \ \forall\vTaskJ\in\mathcal{I}_{\vCurrentRound}\nonumber \\
 & \ \ \ \ \ \ \ \ \ \ \ \bm{\lambda}_{\text{min}}\left(\bm{L}\bm{L}^{\mathsf{T}}\right)\geq p\ \ \text{and}\ \ \bm{\lambda}_{\text{max}}\left(\bm{L}\bm{L}^{\mathsf{T}}\right)\leq q\enspace.\nonumber 
\end{align}

\subsection{Online Formulation}

The optimization problem above can be mapped to the standard online
learning framework by unrolling $\bm{L}$ and $\bm{S}$ into a vector
$\bm{\theta}=[\text{vec}(\bm{L})\ \text{vec}(\bm{S})]^{\mathsf{T}}\in\mathbb{R}^{dk+k\vNumTotalTasks}$. Choosing $\bm{\Omega}_{0}(\bm{\theta}) =\mu_{2}\sum_{i=1}^{dk}\bm{\theta}_{i}^{2}+\mu_{1}\sum_{i=dk+1}^{dk+k\vNumTotalTasks}\bm{\theta}_{i}^{2}\enspace$, and
$\bm{\Omega}_{j}(\bm{\theta}) =\bm{\Omega}_{j-1}(\bm{\theta})+\eta_{\vTaskJ}l_{\vTaskJ}(\bm{\theta})$, we can write the safe lifelong policy search problem (Eq.~\eqref{Eq:RelaxedOne}) as:
\begin{align}
\bm{\theta}_{\vCurrentRound+1} & =\arg\min_{\bm{\theta}\in\mathcal{K}}\bm{\Omega}_{\vCurrentRound}(\bm{\theta})\enspace,\label{Eq:Online}
\end{align}
where $\mathcal{K}\subseteq\mathbb{R}^{dk+k\vNumTotalTasks}$ is the set of allowable policies under the given safety constraints.
Note that the loss for task $\vTaskJ$ can be written as a bilinear
product in $\bm{\theta}$: 
\[
l_{\vTaskJ}(\bm{\theta})=-\frac{1}{n_{\vTaskJ}}\sum_{k=1}^{n_{\vTaskJ}}\sum_{m=0}^{M_{\vTaskJ}-1}\!\!\!\log\!\left[\pi_{\bm{\Theta_{\bm{L}}}\bm{\Theta}_{\bm{s}_{\vTaskJ}}}^{\left(\vTaskJ\right)}\!\!\left(\bm{u}_{m}^{\left(k,\ \vTaskJ\right)}\mid\bm{x}_{m}^{\left(k,\ \vTaskJ\right)}\right)\!\right]
\]
\begin{align*}
\bm{\Theta}_{\bm{L}}=\left[\!\begin{array}{ccc}
\bm{\theta}_{1} & \!\hdots\! & \bm{\theta}_{d(k-1)+1}\\
\vdots & \!\vdots\! & \vdots\\
\bm{\theta}_{d} & \!\hdots\! & \bm{\theta}_{dk}
\end{array}\!\right],\ \bm{\Theta}_{\bm{s}_{\vTaskJ}}=\left[\!\begin{array}{c}
\bm{\theta}_{dk+1}\\
\vdots\\
\bm{\theta}_{(d+1)k+1}
\end{array}\!\right].
\end{align*}

We see that the problem in Eq.~\eqref{Eq:Online} is equivalent
to Eq.~\eqref{Eq:RelaxedOne} by noting
that at $\vCurrentRound$ rounds, $\bm{\Omega}_{\vCurrentRound}=\sum_{j=1}^{\vCurrentRound}\eta_{\vTaskJ}l_{\vTaskJ}(\bm{\theta})+\bm{\Omega}_{0}(\bm{\theta})$.  


\section{Online Learning Method}
\label{Sec:SolutionStrategy} 

We solve Eq.~\eqref{Eq:Online}
in two steps. First, we determine the unconstrained solution 
$\tilde{\bm{\theta}}_{\vCurrentRound+1}$  when $\mathcal{K}=\mathbb{R}^{dk+k\vNumTotalTasks}$ (see Sect.~\ref{Sec:Unconstrained}). Given $\tilde{\bm{\theta}}_{\vCurrentRound+1}$,
we derive the constrained solution 
$\hat{{\bm{\theta}}}_{\vCurrentRound+1}$ by learning a projection 
$\text{Proj}_{\bm{\Omega}_{\vCurrentRound},\mathcal{K}}\left(\tilde{\bm{\theta}}_{\vCurrentRound+1}\right)$ 
to the constraint set $\mathcal{K}\subseteq\mathbb{R}^{dk+k\vNumTotalTasks}$,
which amounts to minimizing the Bregman divergence over $\bm{\Omega}_{\vCurrentRound}(\bm{\theta})$ (see Sect.~\ref{Sec:Constraint})%
\footnote{In Sect.~\ref{Sec:Constraint}, we linearize the
loss around the constrained solution of the previous round
to increase stability and ensure convergence. Given the linear losses,
it suffices to solve the Bregman divergence over the regularizer, reducing the  
computational cost.%
}.  The complete approach is given in Algorithm~\ref{alg:SafeOnlinePG} and is available as a software implementation on the authors' websites.


\subsection{Unconstrained Policy Solution}

\label{Sec:Unconstrained} Although Eq.~\eqref{Eq:RelaxedOne}
is not jointly convex in both $\bm{L}$ and $\bm{S}$, it is separably convex (for log-concave policy distributions). Consequently,
we follow an alternating optimization approach, first computing $\bm{L}$ while holding $\bm{S}$ fixed, and then updating
$\bm{S}$ given the acquired $\bm{L}$. We detail this process for two popular PG learners, eREINFORCE~\cite{Williams92simplestatistical} and eNAC~\cite{Peters20081180}. The derivations of the update rules below can be found in Appendix~\ref{App:Update}. 

These updates are governed by learning rates $\beta$ and $\lambda$ that decay over time; $\beta$ and $\lambda$ can be chosen using line-search methods as discussed by~\citet{Boyd}. In our experiments, we adopt a simple yet effective strategy, where $\beta=c j^{-1}$ and $\lambda = c j^{-1}$, with $0 < c < 1$. 

\textbf{Step 1: Updating $\bm{L}$~~} Holding $\bm{S}$ fixed, the
latent repository can be updated according to: 
\begin{align*}
\bm{L}_{\beta+1} & =\bm{L}_{\beta}-\eta_{\bm{L}}^{\beta}\nabla_{\bm{L}}\bm{e}_{\vCurrentRound}(\bm{L},\bm{S}) \hspace{-2em}& \text{(eREINFORCE)}\\
\bm{L}_{\beta+1} & =\bm{L}_{\beta}-\eta_{\bm{L}}^{\beta}\bm{G}^{-1}(\bm{L}_{\beta},\bm{S}_{\beta})\nabla_{\bm{L}}\bm{e}_{\vCurrentRound}(\bm{L},\bm{S}) \hspace{-2em}& \text{(eNAC)}
\end{align*}
with learning rate $\eta_{\bm{L}}^{\beta}\in\mathbb{R}$,
and $\bm{G}^{-1}(\bm{L},\bm{S})$ as the inverse of the Fisher
information matrix~\cite{Peters20081180}.

In the special case of Gaussian policies, the update for $\bm{L}$
can be derived in a closed form as $\bm{L}_{\beta+1}=\bm{Z}_{\bm{L}}^{-1}\bm{v}_{\bm{L}}$,
where
\begin{align*}
\bm{Z}_{\bm{L}} & = \!2\mu_{2}\bm{I}_{dk\times dk}\!+\!\sum_{j=1}^{\vCurrentRound}\!\frac{\eta_{\vTaskJ}}{n_{\vTaskJ}\sigma_{\vTaskJ}^{2}}\!\sum_{k=1}^{n_{\vTaskJ}}\!\sum_{m=0}^{M_{\vTaskJ}\!-1}\!\!\!\text{vec}\!\left(\!\bm{\Phi}\bm{s}_{\vTaskJ}^{\mathsf{T}}\!\right)\!\!\left(\!\bm{\Phi}^{\mathsf{T}}\!\otimes\!\bm{s}_{\vTaskJ}^{\mathsf{T}}\!\right)\\
\bm{v}_{\bm{L}} & =\sum_{j}\frac{\eta_{\vTaskJ}}{n_{\vTaskJ}\sigma_{\vTaskJ}^{2}}\sum_{k=1}^{n_{\vTaskJ}}\sum_{m=0}^{M_{\vTaskJ}\!-1}\!\!\!\text{vec}\left(\bm{u}_{m}^{\left(k,\ \vTaskJ\right)}\bm{\Phi}\bm{s}_{\vTaskJ}^{\mathsf{T}}\right) \enspace ,
\end{align*}
$\sigma_{\vTaskJ}^{2}$ is the covariance of the Gaussian
policy for a task $\vTaskJ$, and $\bm{\Phi}=\bm{\Phi}\left(\bm{x}_{m}^{\left(k,\ \vTaskJ\right)}\right)$
denotes the state features.

\textbf{Step 2: Updating $\bm{S}$~~} Given the fixed basis $\bm{L}$, the coefficient
matrix $\bm{S}$ is updated column-wise for all $\vTaskJ\in\mathcal{I}_{\vCurrentRound}$:
\begin{align*}
\bm{s}_{\lambda+1}^{(\vTaskJ)} & =\bm{s}_{\lambda+1}^{(\vTaskJ)}-\eta_{\bm{S}}^{\lambda}\nabla_{\bm{s}_{\vTaskJ}}e_{\vCurrentRound}(\bm{L},\bm{S}) \hspace{-3em}& \text{(eREINFORCE)}\\
\bm{s}_{\lambda+1}^{(\vTaskJ)} & =\bm{s}_{\lambda+1}^{(\vTaskJ)}-\eta_{\bm{S}}^{\lambda}\bm{G}^{-1}(\bm{L}_{\beta},\bm{S}_{\beta})\nabla_{\bm{s}_{\vTaskJ}}e_{\vCurrentRound}(\bm{L},\bm{S}) \hspace{-3em}& \text{(eNAC)}
\end{align*}
with learning rate $\eta_{\bm{S}}^{\lambda}\in\mathbb{R}$.  For Gaussian policies, the closed-form of the update
is $\bm{s}_{\vTaskJ}=\bm{Z}_{\bm{s}_{\vTaskJ}}^{-1}\bm{v}_{\bm{s}_{\vTaskJ}}$, where
\begin{align*}
\bm{Z}_{\bm{s}_{\vTaskJ}} & =2\mu_{1}\bm{I}_{k\times k}+\sum_{\vTaskK=\vTaskJ}\frac{\eta_{\vTaskJ}}{n_{\vTaskJ}\sigma_{\vTaskJ}^{2}}\sum_{k=1}^{n_{\vTaskJ}}\sum_{m=0}^{M_{\vTaskJ}\!-1}\!\!\!\bm{L}^{\mathsf{T}}\bm{\Phi}\bm{\Phi}^{\mathsf{T}}\bm{L}\\
\bm{v}_{\vTaskJ} & =\sum_{\vTaskK=\vTaskJ}\frac{\eta_{\vTaskJ}}{n_{\vTaskJ}\sigma_{\vTaskJ}^{2}}\sum_{k=1}^{n_{\vTaskJ}}\sum_{m=0}^{M_{\vTaskJ}\!-1}\!\!\!\bm{u}_{m}^{\left(k,\ \vTaskJ\right)}\bm{L}^{\mathsf{T}}\bm{\Phi} \enspace .
\end{align*}

\subsection{Constrained Policy Solution}

\label{Sec:Constraint} Once we have obtained the unconstrained solution $\tilde{\bm{\theta}}_{\vCurrentRound+1}$ (which
satisfies Eq.~\eqref{Eq:Online}, but can lead to policy parameters
in unsafe regions), we then derive the constrained solution to ensure safe policies.  We learn a projection $\text{Proj}_{\bm{\Omega}_{\vCurrentRound},\mathcal{K}}\left(\tilde{\bm{\theta}}_{\vCurrentRound+1}\right)$ from $\tilde{\bm{\theta}}_{\vCurrentRound+1}$ to the constraint set: 
\begin{equation}
\hat{\bm{\theta}}_{\vCurrentRound+1}=\arg\min_{\bm{\theta}\in\mathcal{K}}\mathcal{B}_{\bm{\Omega}_{\vCurrentRound},\mathcal{K}}\left(\bm{\theta},\tilde{\bm{\theta}}_{\vCurrentRound+1}\right)\label{Eq:Bregman} \enspace ,
\end{equation}
where $\mathcal{B}_{\bm{\Omega}_{\vCurrentRound},\mathcal{K}}\!\left(\!\bm{\theta},\tilde{\bm{\theta}}_{\vCurrentRound+1}\!\right)$
is the Bregman divergence over $\bm{\Omega}_{\vCurrentRound}$: 
\begin{align*}
\mathcal{B}_{\bm{\Omega}_{\vCurrentRound},\mathcal{K}}\!\left(\!\bm{\theta},\tilde{\bm{\theta}}_{\vCurrentRound+1}\right) & =\bm{\Omega}_{\vCurrentRound}(\bm{\theta})-\bm{\Omega}_{\vCurrentRound}(\tilde{\bm{\theta}}_{\vCurrentRound+1})\\
 & \hspace{1em}-\text{trace}\!\left(\!\nabla_{\bm{\theta}}\bm{\Omega}_{\vCurrentRound}\left(\bm{\theta}\right)\Big|_{\tilde{\bm{\theta}}_{\vCurrentRound+1}}\!\left(\bm{\theta}-\tilde{\bm{\theta}}_{\vCurrentRound+1}\!\right)\!\!\right)~.
\end{align*}
Solving Eq.~\eqref{Eq:Bregman} is computationally
expensive since $\bm{\Omega}_{\vCurrentRound}(\bm{\theta})$ includes
the sum back to the original round. To remedy this problem, ensure the 
stability of our approach, and guarantee that the constrained solutions
for all observed tasks lie within a bounded region, we linearize the
current-round loss function $l_{\vTaskR}(\bm{\theta})$ around
the \emph{constrained} solution of the previous round $\hat{\bm{\theta}}_{\vCurrentRound}$:
\begin{equation}
l_{\vTaskR}\left(\hat{\bm{u}}\right)=\hat{\bm{f}}_{\vTaskR}\Big|_{\hat{\bm{\theta}}_{\vCurrentRound}}^{\mathsf{T}}\hat{\bm{u}}\label{Eq:Linear} \enspace ,
\end{equation}
where 
\begin{align*}
\hat{\bm{f}}_{\vTaskR}\Big|_{\hat{\bm{\theta}}_{\vCurrentRound}}&=\left[\!\!\begin{array}{c}
\nabla_{\bm{\theta}}l_{\vTaskR}\left(\bm{\theta}\right)\Big|_{\hat{\bm{\theta}}_{\vCurrentRound}}\\
l_{\vTaskR}\left(\bm{\theta}\right)\Big|_{\hat{\bm{\theta}}_{{\vCurrentRound}}}-\nabla_{\bm{\theta}}l_{\vTaskR}\left(\bm{\theta}\right)\Big|_{\hat{\bm{\theta}}_{{\vCurrentRound}}}\hat{\bm{\theta}}_{\vCurrentRound}
\end{array}\!\!\right], & \hat{\bm{u}}&=\left[\!\begin{array}{c}
\bm{u}\\
1
\end{array}\!\right] ~.
\end{align*}

Given the above linear form, we can rewrite the optimization problem
in Eq.~\eqref{Eq:Bregman} as: 
\begin{align}
\hat{\bm{\theta}}_{\vCurrentRound+1} & =\arg\min_{\bm{\theta}\in\mathcal{K}}\mathcal{B}_{\bm{\Omega}_{0},\mathcal{K}}\left(\bm{\theta},\tilde{\bm{\theta}}_{\vCurrentRound+1}\right)\label{Eq:ConstraintBregman} \enspace .
\end{align}
Consequently, determining \emph{safe policies} for lifelong policy
search reinforcement learning amounts to solving: 
\begin{align*}
\min_{\bm{L},\bm{S}}\  & \mu_{1}\|\bm{S}\|_{\mathsf{F}}^{2}+\bm{\mu}_{2}\|\bm{L}\|_{\mathsf{F}}^{2} \\
&+2\mu_{1}\text{trace}\left({\bm{S}}^{\mathsf{T}}\Big|_{\tilde{\theta}_{r+1}}\bm{S}\right) +2\mu_{2}\text{trace}\left({\bm{L}}\Big|_{\tilde{\theta}_{r+1}}\bm{L}\right)\\
\text{s.t.}\ &\bm{A}_{\vTaskJ}\bm{L}\bm{s}_{\vTaskJ}\leq\bm{b}_{\vTaskJ}\ \ \ \forall\vTaskJ\in\mathcal{I}_{\vCurrentRound}\\
 & \bm{L}\bm{L}^{\mathsf{T}}\leq p\bm{I}\ \ \ \ \text{and}\ \ \ \bm{L}\bm{L}^{\mathsf{T}}\geq q\bm{I} \enspace .
\end{align*}

To solve the optimization problem above, we start by converting the
inequality constraints to equality constraints by introducing slack variables
$\bm{c}_{\vTaskJ}\geq0$. We also guarantee that these slack variables
are bounded by incorporating $\|\bm{c}_{\vTaskJ}\|\leq\bm{c}_{\text{max}},\ \forall\vTaskJ\in\{1,\dots,\vNumTotalTasks\}$:
\begin{align*}
\min_{\bm{L},\bm{S},\bm{C}} \ \ & \mu_{1}\|\bm{S}\|_{\mathsf{F}}^{2}+\mu_{2}\|\bm{L}\|_{\mathsf{F}}^{2}\\
& +2\mu_{2}\text{trace}\left({\bm{L}}^{\mathsf{T}}\Big|_{\tilde{\bm{\theta}}_{r+1}}\bm{L}\right) +2\mu_{1}\text{trace}\left({\bm{S}}^{\mathsf{T}}\Big|_{\tilde{\bm{\theta}}_{r+1}}\bm{S}\right)\\
 \text{s.t.} \ & \bm{A}_{\vTaskJ}\bm{L}\bm{s}_{\vTaskJ}=\bm{b}_{\vTaskJ}-\bm{c}_{\vTaskJ}\ \ \ \forall\vTaskJ\in\mathcal{I}_{\vCurrentRound}\\
 &  \bm{c}_{\vTaskJ}>0\ \ \ \ \text{and}\ \ \ \|\bm{c}_{\vTaskJ}\|_{2}\leq\bm{c}_{\text{max}}\ \ \ \forall\vTaskJ\in\mathcal{I}_{\vCurrentRound}\\
 &\bm{L}\bm{L}^{\mathsf{T}}\leq p\bm{I}\ \ \ \ \text{and}\ \ \ \bm{L}\bm{L}^{\mathsf{T}}\geq q\bm{I} \enspace .
\end{align*}
With this formulation, learning $\text{Proj}_{\bm{\Omega}_{\vCurrentRound},\mathcal{K}}\left(\tilde{\bm{\theta}}_{\vCurrentRound+1}\right)$
amounts to solving second-order cone and semi-definite programs.

\subsubsection{Semi-Definite Program for Learning $\bm{L}$}\label{Sec:SemiDefiniteProgram}

This section determines the constrained projection of the shared 
basis $\bm{L}$ given fixed $\bm{S}$ and $\bm{C}$. We show that $\bm{L}$
can be acquired efficiently, since this step can be relaxed to solving
a semi-definite program in $\bm{L}\bm{L}^{\mathsf{T}}$~\cite{Boyd}.
To formulate the semi-definite program, note that 
\begin{align*}
\text{trace}\!\left(\!{\bm{L}}^{\mathsf{T}}\Big|_{\tilde{\bm{\theta}}_{r+1}}\bm{L}\!\right) & =\sum_{\vTempIndexA=1}^{k}{{\bm{l}}_{r+1}^{(\vTempIndexA)}}^{\!\!\mathsf{T}}\Big|_{\tilde{\bm{\theta}}_{r+1}}\bm{l}_{\vTempIndexA} \\
&\leq\sum_{\vTempIndexA=1}^{k}\left\|{\bm{l}}_{r+1}^{(\vTempIndexA)}\Big|_{\tilde{\bm{\theta}}_{r+1}}\right\|_{2} \left\|\bm{l}_{\vTempIndexA}\right\|_{2}\\
 & \leq\sqrt{\sum_{\vTempIndexA=1}^{k}\left\|{\bm{l}}_{\vCurrentRound}^{(\vTempIndexA)}\Big|_{\tilde{\bm{\theta}}_{r+1}}\right\|_{2}^{2}}\sqrt{\sum_{\vTempIndexA=1}^{k}\left|\left|{\bm{l}}_{\vTempIndexA}\right|\right|_{2}^{2}}\\
 & =\left|\left|{\bm{L}}\Big|_{\tilde{\bm{\theta}}_{r+1}}\right|\right|_{\mathsf{F}}\sqrt{\text{trace}\left(\bm{L}\bm{L}^{\mathsf{T}}\right)} \enspace .
\end{align*}
From the constraint set, we recognize: 
\begin{align*}
\bm{s}_{\vTaskJ}^{\mathsf{T}}\bm{L}^{\mathsf{T}} & =\left(b_{\vTaskJ}-\bm{c}_{\vTaskJ}\right)^{\mathsf{T}}\left(\bm{A}_{\vTaskJ}^{\dagger}\right)^{\mathsf{T}}\\
\implies\bm{s}_{\vTaskJ}^{\mathsf{T}}\bm{L}^{\mathsf{T}}\bm{L}\bm{s}_{\vTaskJ} & =\bm{a}_{\vTaskJ}^{\mathsf{T}}\bm{a}_{\vTaskJ}\ \ \ \ \text{with}\ \ \ \ \ \bm{a}_{\vTaskJ}=\bm{A}_{\vTaskJ}^{\dagger}\left(\bm{b}_{\vTaskJ}-\bm{c}_{\vTaskJ}\right) ~.
\end{align*}
Since $\text{spectrum}\left(\bm{L}\bm{L}^{\mathsf{T}}\right)=\text{spectrum}\left(\bm{L}^{\mathsf{T}}\bm{L}\right)$, 
we can write: 
\begin{align*}
\min_{\bm{X}\subset\mathcal{S}_{++}} & \mu_{2}\text{trace}(\bm{X})+2\mu_{2}\left|\left|{\bm{L}}\Big|_{\tilde{\bm{\theta}}_{r+1}}\right|\right|_{\mathsf{F}}\sqrt{\text{trace}\left(\bm{X}\right)}\\
 \text{s.t.}\ \ &\bm{s}_{\vTaskJ}^{\mathsf{T}}\bm{X}\bm{s}_{\vTaskJ}=\bm{a}_{\vTaskJ}^{\mathsf{T}}\bm{a}_{\vTaskJ}\ \ \ \ \forall \vTaskJ\in\mathcal{I}_{\vCurrentRound}\\
 & \bm{X}\leq p\bm{I}~~~\text{and}~~~\bm{X}\geq q\bm{I}~,~~~\text{with}~~~\bm{X}=\bm{L}^{\mathsf{T}}\bm{L} ~.
\end{align*}

\subsubsection{Second-Order Cone Program for Learning Task Projections}\label{Sec:Cone}

Having determined $\bm{L}$, we can acquire $\bm{S}$ and
update $\bm{C}$ by solving a second-order cone
program~\cite{Boyd} of the following form: 
\begin{align*}
 & \min_{\bm{s}_{\vTaskOne},\dots,\bm{s}_{\vTaskJ},\bm{c}_{\vTaskOne},\dots,\bm{c}_{\vTaskJ}}\mu_{1}\sum_{j=1}^{{\vCurrentRound}}\|\bm{s}_{\vTaskJ}\|_{2}^{2}+2\mu_{1}\sum_{j=1}^{\vCurrentRound}{\bm{s}}_{\vTaskJ}^{\mathsf{T}}\Big|_{\hat{\bm{\theta}}_{\vCurrentRound}}\bm{s}_{\vTaskJ}\\
 & \text{s.t.}\ \ \ \ \ \ \bm{A}_{\vTaskJ}\bm{L}\bm{s}_{\vTaskJ}=\bm{b}_{\vTaskJ}-\bm{c}_{\vTaskJ}\\
 & \ \ \ \ \ \ \ \ \ \ \ \ \ \bm{c}_{\vTaskJ}>0\ \ \ \ \ \|\bm{c}_{\vTaskJ}\|_{2}^{2}\leq\bm{c}_{\text{max}}^{2}\ \forall\vTaskJ\in\mathcal{I}_{\vCurrentRound} \enspace .
\end{align*}

\begin{algorithm}[t]
\caption{Safe Online Lifelong Policy Search} \label{alg:SafeOnlinePG}          
\begin{algorithmic}[1]
\STATE \textbf{Inputs:} Total number of rounds $R$, weighting factor $\eta=\sfrac{1}{\sqrt{R}}$, regularization parameters $\mu_{1}$ and $\mu_{2}$, constraints  $p$ and $q$, number of latent basis vectors $k$. 
\STATE $\bm{S}=\text{zeros}(k, |\mathcal{T}|)$, $\bm{L} = \text{diag}_{k}(\zeta)$ with $p \leq \zeta^{2} \leq q$
\FOR{$j=1$ to $R$}
\STATE $t_{j} \leftarrow \text{sampleTask}()$, and update $\mathcal{I}_{j}$
\STATE Compute \textbf{unconstrained solution} $\tilde{\bm{\theta}}_{j+1}$ (Sect.~\ref{Sec:Unconstrained})
\STATE Fix $\bm{S}$ and $\bm{C}$, and update $\bm{L}$ (Sect.~\ref{Sec:SemiDefiniteProgram})
\STATE Use updated $\bm{L}$ to derive $\bm{S}$ and $\bm{C}$ (Sect.~\ref{Sec:Cone})
\ENDFOR
\STATE \textbf{Output:} Safety-constrained $\bm{L}$ and $\bm{S}$
\end{algorithmic}
\end{algorithm}

\section{Theoretical Guarantees} \label{sect:MainResults}

This section quantifies the performance of our approach by providing
formal analysis of the regret after $\vNumTotalRounds$ rounds. We
show that the safe lifelong reinforcement learner exhibits \emph{sublinear}
regret in the total number of rounds. Formally, we prove the following
theorem: 
\begin{theorem}[Sublinear Regret]\label{Theo:Main} 
After $\vNumTotalRounds$ rounds and choosing $\forall \vTaskJ \in \mathcal{I}_\vNumTotalRounds \  \  \eta_{\vTaskJ}\!=\eta=\frac{1}{\sqrt{\vNumTotalRounds}}$,
$\bm{L}\Big|_{\hat{\bm{\theta}}_{1}}=\text{diag}_{\bm{k}}(\zeta)$,
with $\text{diag}_{\bm{k}}(\cdot)$ being a diagonal matrix among
the $\bm{k}$ columns of $\bm{L}$, $p\leq\zeta^{2}\leq q$, and $\bm{S}\Big|_{\hat{\bm{\theta}}_{1}}=\bm{0}_{k\times \vNumTotalTasks}$,
the safe lifelong reinforcement learner exhibits sublinear regret
of the form: 
\[
\sum_{j=1}^{\vNumTotalRounds}l_{\vTaskJ}\left(\hat{\bm{\theta}}_{j}\right)-l_{\vTaskJ}(\bm{u})=\mathcal{O}\left(\sqrt{\vNumTotalRounds}\right) \text{\ \ for any $\bm{u}\in\mathcal{K}$.}
\]
\end{theorem} 


\textbf{Proof Roadmap:} The remainder of this section completes our proof of Theorem~\ref{Theo:Main}; further details are given in Appendix~\ref{appendix:proofs}. 
We assume linear losses for all tasks in the constrained case in accordance
with Sect.~\ref{Sec:Constraint}. Although linear losses for
policy search RL are too restrictive given a single
operating point, as discussed previously, we remedy this problem by
generalizing to the case of piece-wise linear losses, where the linearization
operating point is a resultant of the optimization problem. 
To bound the regret, we need to bound the dual Euclidean norm
(which is the same as the Euclidean norm) of the gradient of the loss function,
then prove Theorem~\ref{Theo:Main} by bounding: (1) task $\vTaskJ$'s
gradient loss (Sect.~\ref{Sec:itGradient}), and (2) linearized
losses with respect to $\bm{L}$ and $\bm{S}$ (Sect.~\ref{Sec:LinearBound}).

\subsection{Bounding $\vTaskJ$'s Gradient Loss}
\label{Sec:itGradient} 

We start by stating essential lemmas for Theorem~\ref{Theo:Main}; due to space constraints, proofs for all lemmas are available in the supplementary material.
Here, we bound the gradient of a loss function $l_{\vTaskJ}(\bm{\theta})$
at round $\vCurrentRound$ under Gaussian policies%
\footnote{Please note that derivations for other forms of log-concave policy
distributions could be derived in similar manner. In this work, we
focus on Gaussian policies since they cover a broad spectrum of real-world
applications.}.  

\begin{assumption}\label{Ass:Pol} 
We assume that the policy for a task $\vTaskJ$ is Gaussian, the action set $\mathcal{U}$ is bounded by $\bm{u}_{\max}$,
and the feature set is upper-bounded by $\bm{\Phi}_{\max}$. 
\end{assumption}


\begin{lemma}\label{Lemma:Lemma1} Assume task $\vTaskJ$'s policy at round $\vCurrentRound$ is given by 
$\pi_{\bm{\alpha}_{\vTaskJ}}^{\left(\vTaskJ\right)}\!\left(\!\bm{u}_{m}^{\left(k,\ \vTaskJ\right)}|\bm{x}_{m}^{\left(k,\ \vTaskJ\right)}\!\right)\!\Big|_{\hat{\bm{\theta}}_{\vCurrentRound}}\!=\mathcal{N}\!\left(\!\bm{\alpha}_{\vTaskJ}^{\mathsf{T}}\Big|_{\hat{\bm{\theta}}_{\vCurrentRound}}\bm{\Phi}\left(\!\bm{x}_{m}^{\left(k,\ \vTaskJ\right)}\right)\!,\bm{\sigma}_{\vTaskJ}\!\right)$, 
for states $\bm{x}_{m}^{\left(k,\ \vTaskJ\right)}\in\mathcal{X}_{\vTaskJ}$
and actions $\bm{u}_{m}^{\left(k,\ \vTaskJ\right)}\in\mathcal{U}_{\vTaskJ}$. 
For \\
$l_{\vTaskJ}\!\!\left(\bm{\alpha}_{\vTaskJ}\right)=-\frac{1}{n_{\vTaskJ}}\displaystyle\sum_{k=1}^{n_{\vTaskJ}}\!\sum_{m=0}^{M_{\vTaskJ}-1}\!\!\!\log\left[\pi_{\bm{\alpha}_{\vTaskJ}}^{\left(\vTaskJ\right)}\!\left(\!\bm{u}_{m}^{\left(k,\ \vTaskJ\right)}|\bm{x}_{m}^{\left(k,\ \vTaskJ\right)}\!\right)\!\right]$,~the gradient $\nabla_{\!\bm{\alpha}_{\vTaskJ}}l_{\vTaskJ}\!\!\left(\!\bm{\alpha}_{\vTaskJ}\!\right)\!\Big|_{\hat{\bm{\theta}}_{\vCurrentRound}}$ satisfies $\left|\left|\nabla_{\!\bm{\alpha}_{\vTaskJ}}l_{\vTaskJ}\!\!\left(\!\bm{\alpha}_{\vTaskJ}\!\right)\!\Big|_{\hat{\bm{\theta}}_{\vCurrentRound}}\right|\right|_{2}\leq$
\begin{align*}
 \frac{M_{\vTaskJ}}{\sigma_{\vTaskJ}^{2}}\Bigg(\!u_{\max}+\!\!\!\max_{\vTaskK\in\mathcal{I}_{\vCurrentRound-1}}\!\!\!\left\{ \left|\left|\bm{A}_{\vTaskK}^{+}\right|\right|_{2}\left(\left|\left|\bm{b}_{\vTaskK}\right|\right|_{2}+\bm{c}_{\max}\right)\!\right\} \bm{\Phi}_{\max}\!\Bigg)\bm{\Phi}_{\max}
\end{align*}
for all trajectories and all tasks, with $u_{\max}\!=\!\displaystyle\max_{k,m}\left\{ \left|\bm{u}_{m}^{\left(k,\ \vTaskJ\right)}\right|\right\}$ and $\bm{\Phi}_{\max}\!=\!\displaystyle\max_{k,m}\left\{ \left|\left|\bm{\Phi}\left(\!\bm{x}_{m}^{\left(k,\ \vTaskJ\!\right)}\right)\right|\right|_{2}\right\}$.
\end{lemma} 


\subsection{Bounding Linearized Losses }
\label{Sec:LinearBound} 

As discussed previously, we linearize the
loss of task $\vTaskR$ 
around the constraint
solution of the previous round $\hat{\bm{\theta}}_{\vCurrentRound}$.
To acquire the regret bounds in Theorem~\ref{Theo:Main}, the next
step is to bound the dual norm, $\left\|\hat{\bm{f}}_{\vTaskR}\Big|_{\hat{\bm{\theta}}_{\vCurrentRound}}\right\|_{2}^{\star}=\left\|\hat{\bm{f}}_{\vTaskR}\Big|_{\hat{\bm{\theta}}_{\vCurrentRound}}\right\|_{2}$
of Eq.~\eqref{Eq:Linear}. It can be easily seen 
\begin{align}
\left\|\hat{\bm{f}}_{\vTaskR}\Big|_{\hat{\bm{\theta}}_{\vCurrentRound}}\right\|_{2} & \leq\underbrace{\left|l_{\vTaskR}\left(\bm{\theta}\right)\Big|_{\hat{\theta}_{\vCurrentRound}}\right|}_{\text{constant}}+\underbrace{\left\|\nabla_{\bm{\theta}}l_{\vTaskR}\left(\bm{\theta}\right)\Big|_{\hat{\bm{\theta}}_{\vCurrentRound}}\right\|_{2}}_{\text{Lemma~\ref{Lemma:GradientOne}}}\label{Eq:fhat}\\
 & \hspace{4em}+\left\|\nabla_{\bm{\theta}}l_{\vTaskR}(\bm{\theta})\Big|_{\hat{\bm{\theta}}_{\vCurrentRound}}\right\|_{2}\times\underbrace{\left\|\hat{\bm{\theta}}_{\vCurrentRound}\right\|_{2}}_{\text{Lemma~\ref{Lemma:Theta}}}\nonumber  \enspace .
\end{align}

Since $\left|l_{\vTaskR}\left(\bm{\theta}\right)\Big|_{\hat{\bm{\theta}}_{\vCurrentRound}}\right|$
can be bounded by $\bm{\delta}_{l_{\vTaskR}}$ (see Sect.~\ref{Sec:Background}),\\[-0.5em]
the next step is to bound $\left\|\nabla_{\bm{\theta}}l_{\vTaskR}\left(\bm{\theta}\right)\Big|_{\hat{\bm{\theta}}_{\vCurrentRound}}\right\|_{2}$,
and $\|\hat{\bm{\theta}}_{\vCurrentRound}\|_{2}$.


\begin{lemma}\label{Lemma:GradientOne} The norm of the gradient
of the loss function evaluated at $\hat{\bm{\theta}}_{\vCurrentRound}$
satisfies 
\begin{align*}
 & \left|\left|\nabla_{\bm{\theta}}l_{\vTaskR}\left(\bm{\theta}\right)\Big|_{\hat{\bm{\theta}}_{\vCurrentRound}}\right|\right|_{2}^{2}\leq\Big|\Big|\nabla_{\bm{\alpha}_{\vTaskR}}l_{\vTaskR}\left(\bm{\theta}\right)\Big|_{\hat{\bm{\theta}}_{\vCurrentRound}}\Big|\Big|_{2}^{2}\Bigg(q\times d\\
 & \hspace{0em}\left(\sfrac{2d}{p^{2}}\max_{\vTaskK\in\mathcal{I}_{\vCurrentRound-1}}\left\{ \left|\left|\bm{A}_{\vTaskK}^{\dagger}\right|\right|_{2}^{2}\left(\left|\left|\bm{b}_{\vTaskK}\right|\right|_{2}^{2}+\bm{c}_{\text{max}}^{2}\right)\right\} +1\right)\!\Bigg) \enspace .
\end{align*}
\end{lemma} 
To finalize the bound of $\left\|\hat{\bm{f}}_{\vTaskR}\Big|_{\hat{\bm{\theta}}_{\vCurrentRound}}\right\|_{2}$
as needed for deriving the regret, we must derive an upper-bound for
$\|\hat{\bm{\theta}}_{\vCurrentRound}\|_{2}$: 
\begin{lemma}\label{Lemma:Theta}
The L$_{2}$ norm of the constraint solution at round $\vCurrentRound-1$,
$\|\hat{\bm{\theta}}_{\vCurrentRound}\|_{2}^{2}$ is bounded by
\begin{align*}
\|\hat{\bm{\theta}}_{\vCurrentRound}\|_{2}^{2} & \leq q\times d\Bigg[1+\left|\mathcal{I}_{\vCurrentRound-1}\right|\frac{1}{p^{2}}\\[-0.5em]
 & \hspace{3em}\max_{\vTaskK\in\mathcal{I}_{\vCurrentRound-1}}\left\{ \left|\left|\bm{A}_{\vTaskK}^{\dagger}\right|\right|_{2}^{2}\left(\left|\left|\bm{b}_{\vTaskK}\right|\right|_{2}+\bm{c}_{\text{max}}\right)^{2}\right\} \Bigg] \enspace ,
\end{align*}
where $\left|\mathcal{I}_{\vCurrentRound-1}\right|$ is the number of
unique tasks observed so far. 
\end{lemma} 
Given the previous two lemmas, we can prove the bound
for $\left\|\hat{\bm{f}}_{\vTaskR}\Big|_{\hat{\bm{\theta}}_{\vCurrentRound}}\right\|_{2}$:
\begin{lemma} The L$_{2}$ norm of the linearizing
term of $l_{\vTaskR}(\bm{\theta})$ around $\hat{\bm{\theta}}_{\vCurrentRound}$,
$\left\|\hat{\bm{f}}_{\vTaskR}\Big|_{\hat{\bm{\theta}}_{\vCurrentRound}}\right\|_{2}$,
is bounded by 
\begin{align}
\hspace{-.6em}\left\|\hat{\bm{f}}_{\vTaskR}\Big|_{\hat{\bm{\theta}}_{\vCurrentRound}}\right\|_{2} \!\!& \leq \left\|\nabla_{\bm{\theta}}l_{\vTaskR}\!(\bm{\theta})\Big|_{\hat{\bm{\theta}}_{\vCurrentRound}}\right\|_{2}\!\!\left(\!1\!+\!\|\hat{\bm{\theta}}_{\vCurrentRound}\|_{2}\!\right)+\left|l_{\vTaskR}\!(\bm{\theta})\Big|_{\hat{\bm{\theta}}_{\vCurrentRound}}\right|\label{Eq:fHatFinal}\\
 & \leq\bm{\gamma}_{1}(\vCurrentRound)\left(1+\bm{\gamma}_{2}(\vCurrentRound)\right)+\bm{\delta}_{l_{\vTaskR}}\nonumber \enspace ,
\end{align}
where $\bm{\delta}_{l_{\vTaskR}}$ is the constant upper-bound
on $\left|l_{\vTaskR}(\bm{\theta})\Big|_{\hat{\bm{\theta}}_{\vCurrentRound}}\right|$,
and 
\begin{align*}
 & \bm{\gamma}_{1}(\vCurrentRound)=\frac{1}{n_{\vTaskJ}\sigma_{\vTaskJ}^{2}}\Bigg[\Bigg(u_{\text{max}}\\
 & \hspace{1em}+\max_{\vTaskK\in\mathcal{I}_{\vCurrentRound-1}}\left\{ \left|\left|\bm{A}_{\vTaskK}^{+}\right|\right|_{2}\left(\left|\left|\bm{b}_{\vTaskK}\right|\right|_{2}+\bm{c}_{\text{max}}\right)\right\} \bm{\Phi}_{\text{max}}\Bigg)\bm{\Phi}_{\text{max}}\Bigg]\\
 & \hspace{1em}\times\left(\frac{d}{p}\sqrt{2q}\sqrt{\!\!\max_{\vTaskK\in\mathcal{I}_{\vCurrentRound-1}}\!\!\left\{\! \|\bm{A}_{\vTaskK}^{\dagger}\|_{2}^{2}\left(\|\bm{b}_{\vTaskK}\|_{2}^{2}+\bm{c}_{\text{max}}^{2}\right)\!\right\} }\!+\!\sqrt{qd}\right)\\
 & \bm{\gamma}_{2}(\vCurrentRound)\leq\sqrt{q\times d}\\
 & +\sqrt{\left|\mathcal{I}_{\vCurrentRound-1}\right|}\sqrt{\!1\!+\!\frac{1}{p^{2}}\max_{\vTaskK\in\mathcal{I}_{\vCurrentRound-1}}\!\!\left\{ \left|\left|\bm{A}_{\vTaskK}^{\dagger}\right|\right|_{2}^{2}\!\!\left(\left|\left|\bm{b}_{\vTaskK}\right|\right|_{2}+\bm{c}_{\text{max}}\right)^{2}\right\} }~.
\end{align*}
\end{lemma} 

\subsection{Completing the Proof of Sublinear Regret}

Given the lemmas in the previous section, we now can derive the
sublinear regret bound given in Theorem~\ref{Theo:Main}.
Using results developed by~\citet{AYBKSSz13}, it is easy to see that
\[
\nabla_{\bm{\theta}}\bm{\Omega}_{0}\left(\tilde{\bm{\theta}}_{j}\right)-\nabla_{\bm{\theta}}\bm{\Omega}_{0}\left(\tilde{\bm{\theta}}_{j+1}\right)=\eta_{\vTaskJ}\hat{\bm{f}}_{\vTaskJ}\Big|_{\hat{\bm{\theta}}_{j}} \enspace .
\]
From the convexity of the regularizer, we obtain: 
\begin{align*}
\bm{\Omega}_{0}\left(\hat{\bm{\theta}}_{j}\right)\geq\bm{\Omega}_{0}\left(\hat{\bm{\theta}}_{j+1}\right) & +\left\langle \nabla_{\bm{\theta}}\bm{\Omega}_{0}\left(\hat{\bm{\theta}}_{j+1}\right),\hat{\bm{\theta}}_{j}-\hat{\bm{\theta}}_{j+1}\right\rangle \\
 & +\frac{1}{2}\left|\left|\hat{\bm{\theta}}_{j}-\hat{\bm{\theta}}_{j+1}\right|\right|_{2}^{2} \enspace .
\end{align*}
We have: 
\[
\left\|\hat{\bm{\theta}}_{j}-\hat{\bm{\theta}}_{j+1}\right\|_{2}\leq\eta_{\vTaskJ}\left\|\hat{\bm{f}}_{\vTaskJ}\Big|_{\hat{\bm{\theta}}_{j}}\right\|_{2} \enspace .
\]
Therefore, for any $\bm{u}\in\mathcal{K}$ 
\begin{align*}
\sum_{j=1}^{\vCurrentRound}\eta_{\vTaskJ}\!\left(\!l_{\vTaskJ}\!\left(\hat{\bm{\theta}}_{j}\right)-l_{\vTaskJ}(\bm{u})\!\right) & \leq\sum_{j=1}^{\vCurrentRound}\eta_{\vTaskJ}\left\|\hat{\bm{f}}_{\vTaskJ}\Big|_{\hat{\bm{\theta}}_{j}}\right\|_{2}^{2}\\
 & \hspace{1.5em}+\bm{\Omega}_{0}(\bm{u})-\bm{\Omega}_{0}(\hat{\bm{\theta}}_{1}) \enspace .
\end{align*}
Assuming that $\forall \vTaskJ \ \eta_{\vTaskJ}=\eta$, we can derive: 
\begin{align*}
\sum_{j=1}^{\vCurrentRound}\left(l_{\vTaskJ}\left(\hat{\bm{\theta}}_{j}\right)-l_{\vTaskJ}(\bm{u})\right) & \leq\eta\sum_{j=1}^{\vCurrentRound}\left\|\hat{\bm{f}}_{\vTaskJ}\Big|_{\hat{\bm{\theta}}_{j}}\right\|_{2}^{2}\\
 & \hspace{2em}+\sfrac{1}{\eta}\left(\bm{\Omega}_{0}(\bm{u})-\bm{\Omega}_{0}(\hat{\bm{\theta}}_{1})\right) \enspace .
\end{align*}

The following lemma finalizes the proof of Theorem~\ref{Theo:Main}:
\begin{lemma} After $\vNumTotalRounds$ rounds with $\forall \vTaskJ \ \eta_{\vTaskJ}=\eta=\frac{1}{\sqrt{\vNumTotalRounds}}$,
for any $\bm{u}\in\mathcal{K}$ we have that $\sum_{j=1}^{\vNumTotalRounds}l_{\vTaskJ}(\hat{\bm{\theta}}_{j})-l_{\vTaskJ}(\bm{u})\leq\mathcal{O}\left(\sqrt{\vNumTotalRounds}\right)$.
\end{lemma} 
\begin{proof} From Eq.~\eqref{Eq:fHatFinal}, it follows that 
\begin{align*}
 & \left\|\hat{\bm{f}}_{\vTaskJ}\Big|_{\hat{\bm{\theta}}_{\vCurrentRound}}\right\|_{2}^{2}\leq\bm{\gamma}_{3}(\vNumTotalRounds)+4\bm{\gamma}_{1}^{2}(\vNumTotalRounds)\bm{\gamma}_{2}^{2}(\vNumTotalRounds)\\
 & \hspace{4em}\leq\bm{\gamma}_{3}(\vNumTotalRounds)+8\frac{d}{p^{2}}\bm{\gamma}_{1}^{2}(\vNumTotalRounds)qd \Bigg( 1 +\left|\mathcal{I}_{\vNumTotalRounds-1}\right|\\
 & \hspace{4em}\times\max_{\vTaskK\in\mathcal{I}_{\vNumTotalRounds-1}}\left\{ \|\bm{A}_{\vTaskK}^{\dagger}\|_{2}\left(\|\bm{b}_{\vTaskK}\|_{2}+\bm{c}_{\text{max}}\right)^{2}\right\} \Bigg)
\end{align*}
with $\bm{\gamma}_{3}(\vNumTotalRounds)=4\bm{\gamma}_{1}^{2}(\vNumTotalRounds)+2\max_{\vTaskJ\in\mathcal{I}_{\vNumTotalRounds-1}}\bm{\delta}_{\vTaskJ}^{2}$. 
Since $|\mathcal{I}_{\vNumTotalRounds-1}|\leq\vNumTotalTasks$, we have that 
$\left\|\hat{\bm{f}}_{\vTaskJ}\Big|_{\hat{\theta}_{\vCurrentRound}}\right\|_{2}^{2}\leq\bm{\gamma}_{5}(\vNumTotalRounds)\vNumTotalTasks$
with 
$\displaystyle\bm{\gamma}_{5}=8\sfrac{d}{p^{2}}q\bm{\gamma}_{1}^{2}(\vNumTotalRounds)\max_{\vTaskK\in\mathcal{I}_{\vNumTotalRounds-1}}\left\{ \|\bm{A}_{\vTaskK}^{\dagger}\|_{2}^{2}\left(\|\bm{b}_{\vTaskK}\|_{2}+\bm{c}_{\text{max}}\right)^{2}\right\}$.

Given that $\bm{\Omega}_{0}(\bm{u})\leq qd+\bm{\gamma}_{5}(\vNumTotalRounds)\vNumTotalTasks$,
with $\bm{\gamma}_{5}(\vNumTotalRounds)$ being a constant, we have:
\begin{align*}
\sum_{j=1}^{\vCurrentRound}\!\left(\!l_{\vTaskJ}\!\left(\!\hat{\bm{\theta}}_{j}\!\right)\!-\!l_{\vTaskJ}\!(\bm{u})\!\right) \leq \ & \eta\sum_{j=1}^{\vCurrentRound}\bm{\gamma}_{5}(\vNumTotalRounds)\vNumTotalTasks\\
 &\!\!+\frac{1}{\eta}\!\left(qd+\bm{\gamma}_{5}(\vNumTotalRounds)\vNumTotalTasks-\bm{\Omega}_{0}(\hat{\bm{\theta}}_{1})\!\right)~.
\end{align*}
\textbf{Initializing $\bm{L}$ and $\bm{S}$:} We initialize $\bm{L}\Big|_{\hat{\bm{\theta}}_{1}}\!=\text{diag}_{\bm{k}}(\zeta)$,
with $p\leq\zeta^{2}\leq q$ and $\bm{S}\Big|_{\hat{\bm{\theta}}_{1}}\!=\bm{0}_{k\times\vNumTotalTasks}$ to ensure the invertibility of $\bm{L}$ and that the constraints are
met. This leads to 
\begin{align*}
\sum_{j=1}^{\vCurrentRound}\!\left(\!l_{\vTaskJ}\!\left(\!\hat{\bm{\theta}}_{j}\!\right)\!-\!l_{\vTaskJ}\!(\bm{u})\!\right) \leq \ &\eta\sum_{j=1}^{\vCurrentRound}\bm{\gamma}_{5}(\vNumTotalRounds)\vNumTotalTasks\\
 &\!\!+\sfrac{1}{\eta}\left(qd+\bm{\gamma}_{5}(\vNumTotalRounds)\vNumTotalTasks-\mu_{2}k\zeta\right) ~.
\end{align*}
Choosing $\forall \vTaskJ \ \eta_{\vTaskJ}=\eta=\sfrac{1}{\sqrt{\vNumTotalRounds}}$,
we acquire sublinear regret, finalizing the statement of Theorem~\ref{Theo:Main}:
\begin{align*}
\sum_{j=1}^{\vCurrentRound}\!\left(\!l_{\vTaskJ}\!\left(\!\hat{\bm{\theta}}_{j}\!\right)\!-\!l_{\vTaskJ}\!(\bm{u})\!\right) & \leq\sfrac{1}{\sqrt{\vNumTotalRounds}}\bm{\gamma}_{5}(\vNumTotalRounds)\vNumTotalTasks \vNumTotalRounds\\
 & \hspace{1em}+\sqrt{\vNumTotalRounds}\left(qd+\bm{\gamma}_{5}(\vNumTotalRounds)\vNumTotalTasks-\mu_{2}k\zeta\right)\\
 & \hspace{-3em}\leq\sqrt{\vNumTotalRounds}\Big(\bm{\gamma}_{5}(\vNumTotalRounds)\vNumTotalTasks+qd\bm{\gamma}_{5}(\vNumTotalRounds)\vNumTotalTasks -\mu_{2}k\zeta\Big)\\
 & \hspace{-3em}\leq\mathcal{O}\left(\sqrt{\vNumTotalRounds}\right) \enspace . \\[-2.5em]
\end{align*}
\end{proof}



\section{Experimental Validation}

To validate the empirical performance of our method, we applied our safe online PG algorithm to learn multiple consecutive control tasks on three dynamical systems (Figure~\ref{Fig:Dynamical}).  To generate multiple tasks, we varied the parameterization of each system, yielding a set of control tasks from each domain with varying dynamics.  The optimal control policies for these systems vary widely with only minor changes in the system parameters, providing substantial diversity among the tasks within a single domain.
\begin{figure}[h!]
\centering 
\vspace{-.9em}
\includegraphics[trim = 10mm 130mm 10mm 65mm, clip, width=\columnwidth]{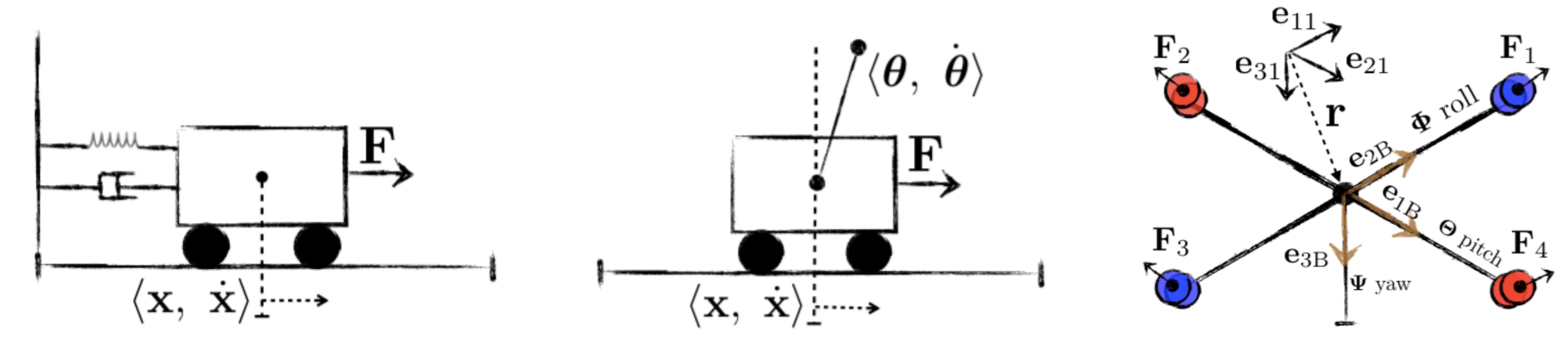}
\vspace{-2em}
\caption{Dynamical systems used in the experiments: \textit{a)} simple mass system (left), \textit{b)} cart-pole (middle), and \textit{c)} quadrotor unmanned aerial vehicle (right).}
\label{Fig:Dynamical}
\vspace{-0.5em}
\end{figure}

\textbf{Simple Mass Spring Damper:~~} The simple mass (SM) system is
characterized by three parameters: the spring constant $k$ in N/m,
the damping constant $d$ in Ns/m and the mass $m$ in kg. The system's
state is given by the position $\bm{x}$ and $\dot{\bm{x}}$ of the
mass, which varies according to a linear force $\bm{F}$. The goal is to train a policy for controlling the mass in a specific state $\bm{g}_{\text{ref}}=\langle\bm{x}_{\text{ref}},\dot{\bm{x}}_{\text{ref}}\rangle$.

 \textbf{Cart Pole:~~} The cart-pole (CP) has been used extensively
as a benchmark for evaluating RL methods~\cite{Busoniu}. CP dynamics
are characterized by the cart's mass $m_{c}$ in kg, the pole's mass
$m_{p}$ in kg, the pole's length in meters, and a damping parameter
$d$ in Ns/m. The state is given by the cart's position $\bm{x}$
and velocity $\dot{\bm{x}}$, as well as the pole's angle $\bm{\theta}$
and angular velocity $\dot{\bm{\theta}}$. The goal is to
train a policy that controls  the pole in an upright position.

\subsection{Experimental Protocol}

We generated 10 tasks for each domain by varying the system 
parameters to ensure a variety of tasks with diverse optimal policies, 
including those with highly chaotic dynamics that are difficult to control. 
We ran each experiment for a total of $\vNumTotalRounds$ rounds, varying from $150$ for the simple mass to $10,000$ for the quadrotor to train $\bm{L}$ and $\bm{S}$, as well as for updating the PG-ELLA and PG models. At each round $j$, the learner observed a task $\vTaskJ$ through 50 trajectories of 150 steps and updated $\bm{L}$ and $\bm{s}_{\vTaskJ}$.   The dimensionality $k$ of the latent space was chosen independently for each domain via cross-validation over 3 tasks, and the learning step size for 
each task domain was determined by a line search after gathering 10 trajectories of length 150.  We used eNAC, a standard PG algorithm, as the base learner.




We compared our approach to both standard PG (i.e., eNAC) and PG-ELLA~\cite{BouAmmar2014Online}, examining both the constrained and unconstrained variants of our algorithm.  We also varied the number of iterations in our alternating optimization from $10$ to $100$ to evaluate the effect of these inner iterations on the performance, as shown in Figures~\ref{Fig:ResBenchmark}~and~\ref{fig:Quad}.  For the two MTL algorithms (our approach and PG-ELLA), the policy parameters for each task $\vTaskJ$ were initialized
using the learned basis (i.e., $\bm{\alpha}_{\vTaskJ}=\bm{L}\bm{s}_{\vTaskJ}$).  We configured PG-ELLA as described by \citet{BouAmmar2014Online}, ensuring a fair comparison.  For the standard PG learner, we provided additional trajectories in order to ensure a fair comparison, as described below.

For the experiments with policy constraints, we generated a set of constraints $(\bm{A}_\vTask, \bm{b}_\vTask)$ for each task that restricted the policy parameters to pre-specified ``safe'' regions, as shown in Figures~\ref{fig:TrajSM} and~\ref{fig:TrajCP}.
We also tested different values for the constraints on $\bm{L}$, varying $p$ and $q$ between $0.1$ to
$10$; our approach showed robustness against this broad range, yielding similar average cost performance.

\subsection{Results on Benchmark Systems}

Figure~\ref{Fig:ResBenchmark} reports our results on the benchmark
simple mass and cart-pole systems. Figures~\ref{fig:PerfSM} and~\ref{fig:PerfCP}
depicts the performance of the learned policy in a lifelong learning
setting over consecutive unconstrained tasks, averaged over all 10 systems over 100 different initial conditions.  These results demonstrate that our approach is capable of outperforming both standard PG (which was provided with 50 {\em additional} trajectories each iteration to ensure a more fair comparison) and PG-ELLA, both in terms of initial performance and learning speed. These figures also show that the performance of our method increases as it is given more alternating iterations per-round for fitting $\bm{L}$ and $\bm{S}$.
\begin{figure*}[t!]
\centering
\vspace{-.5em}
\subfigure[Simple Mass]{
	\label{fig:PerfSM}
\includegraphics[width=0.25\textwidth,height=1.25in,clip,trim=0.0in 0.00in 0.65in 0.4in]{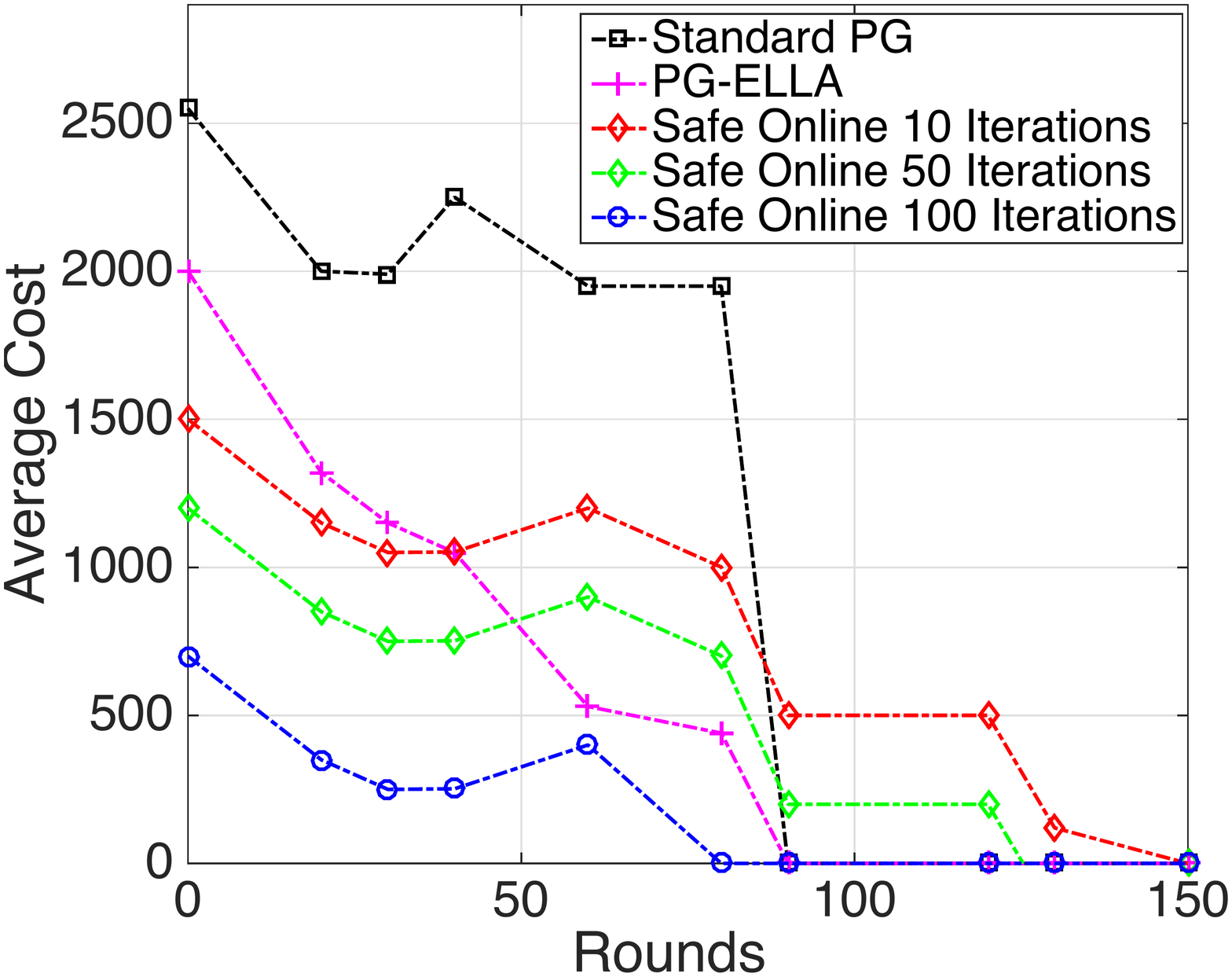}
}
\hfill\hspace{-1em}\hfill
\subfigure[Cart Pole]{
	\label{fig:PerfCP}
\includegraphics[width=0.25\textwidth,height=1.25in,clip,trim=0.0in 0.05in 0.55in 0.4in]{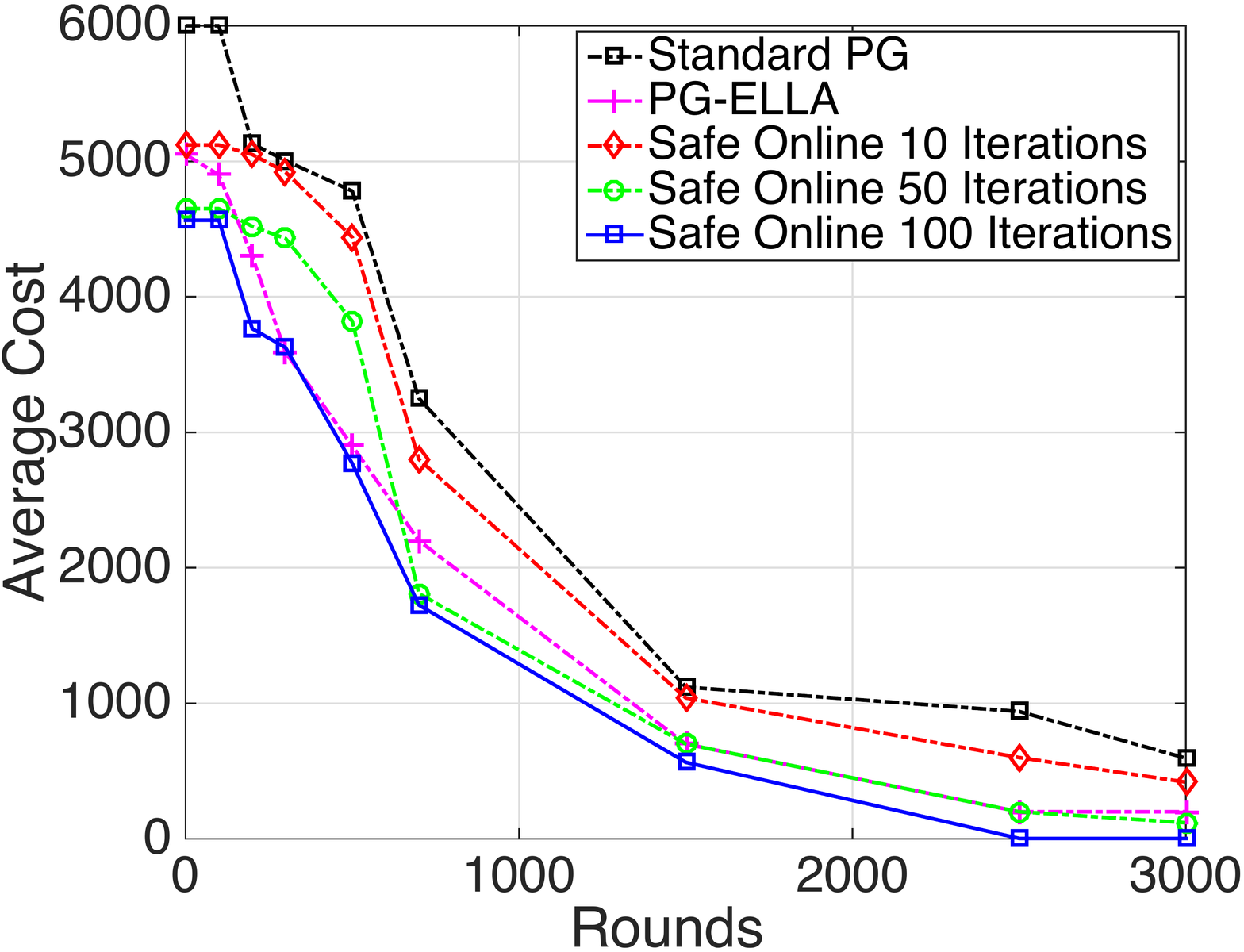}
}
\hfill\hspace{-1em}
\subfigure[Trajectory Simple Mass]{
	\label{fig:TrajSM}
\includegraphics[width=0.25\textwidth,height=1.25in,clip,trim=0.75in 0.35in 0.5in 0.58in]{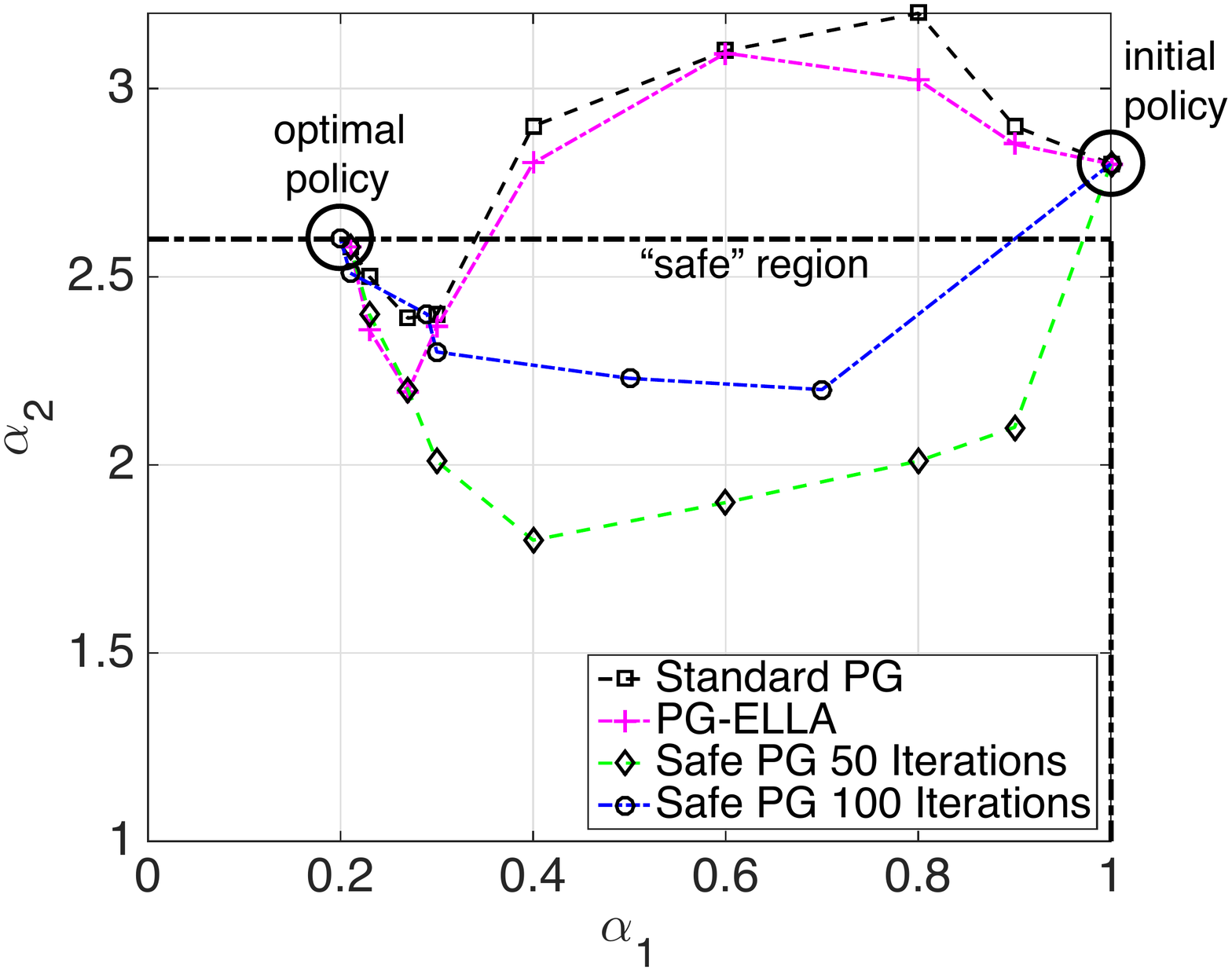}
}
\hfill
\hfill\hspace{-1em}\hfill
\subfigure[Trajectory Cart Pole]{
	\label{fig:TrajCP}
\includegraphics[width=0.25\textwidth,height=1.25in,clip,trim=0.75in 0.35in 0.5in 0.58in]{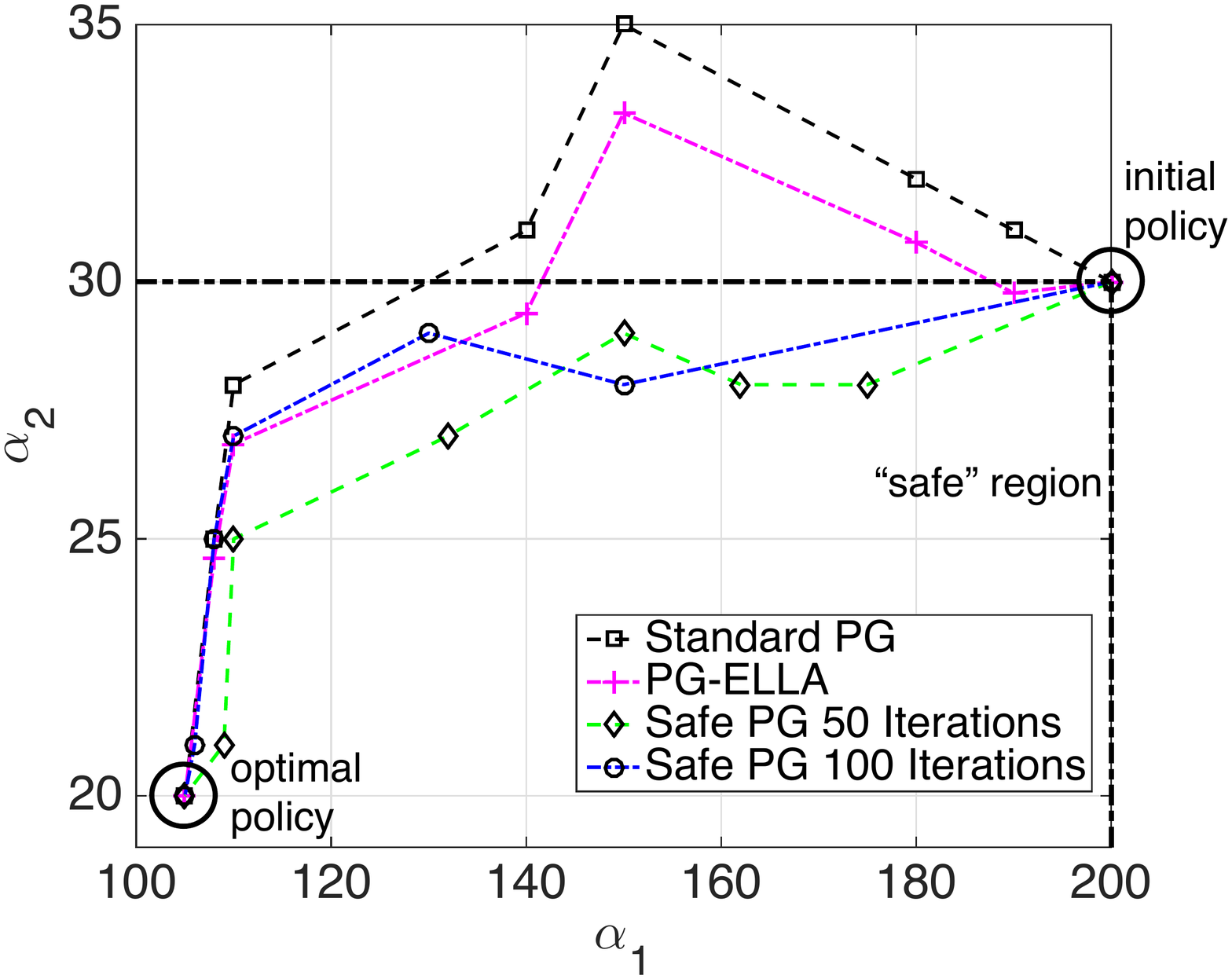}
}
\hfill
\vspace{-1.2em}
\caption{Results on benchmark simple mass and cart-pole systems.  Figures (a) and (b) depict performance in lifelong learning scenarios over consecutive unconstrained tasks, showing that our approach outperforms standard PG and PG-ELLA.  Figures (c) and (d) examine the ability of these method to abide by safety constraints on sample constrained tasks, depicting two dimensions of the policy space ($\alpha_1$ vs $\alpha_2$) and demonstrating that our approach abides by the constraints (the dashed black region).}
\label{Fig:ResBenchmark}
\vspace{-.6em}
\end{figure*}

We evaluated the ability of these methods to respect safety constraints, as shown in 
Figures~\ref{fig:TrajSM} and~\ref{fig:TrajCP}. The thicker black lines in each figure depict the allowable ``safe'' region of the policy space.  To enable online
learning per-task, the same task $\vTaskJ$ was observed on each
round and the shared basis $\bm{L}$ and coefficients 
$\bm{s}_{\vTaskJ}$ were updated using alternating optimization. We then plotted the change in the policy parameter
vectors per iterations (i.e., $\bm{\alpha}_{\vTaskJ}=\bm{L}\bm{s}_{\vTaskJ}$) for each method, demonstrating that our approach abides by the safety constraints,
while standard PG and PG-ELLA can violate them (since they only solve an unconstrained optimization problem). In addition, these figures show that increasing the number of alternating iterations in our method causes it to take a more direct path to the optimal solution.



\subsection{Application to Quadrotor Control}

We also applied our approach to the more challenging domain of quadrotor control. The dynamics of the quadrotor system (Figure~\ref{Fig:Dynamical}) are influenced by inertial
constants around $\bm{e}_{1,B}$, $\bm{e}_{2,B}$, and $\bm{e}_{3,B}$,
thrust factors influencing how the rotor's speed affects the overall
variation of the system's state, and the lengths of the rods supporting
the rotors. Although the overall state of the system can be described
by a 12-dimensional vector, we focus on stability and so consider
only six of these state-variables. The quadrotor system has a high-dimensional
action space, where the goal is to control the four rotational velocities
$\{w_{i}\}_{i=1}^{4}$ of the rotors to stabilize the system. To ensure
realistic dynamics, we used the simulated model described by~\cite{Bouabdallah07designand,Haitham2010b},
which has been verified and used in the control of physical quadrotors.

We generated 10 different quadrotor systems by varying the inertia around the
x, y and z-axes. We used a linear quadratic regulator, as described
by~\citet{Bouabdallah07designand}, to initialize the policies in both the learning
and testing phases. We followed a similar experimental procedure to
that discussed above to update the models.

Figure~\ref{fig:Quad} shows the performance of the unconstrained solution
as compared to standard PG and PG-ELLA. Again, our approach clearly outperforms
standard PG and PG-ELLA in both the initial performance
and learning speed. We also evaluated constrained tasks in a similar manner, again showing that our approach is capable of respecting constraints.  Since the policy space is higher dimensional, we cannot visualize it as well as the benchmark systems, and so instead report the number of iterations it takes our approach to project the policy into the safe region.  Figure~\ref{Fig:NumObs}
shows that our approach requires only one observation of the task to
acquire safe policies, which is substantially lower then standard PG or PG-ELLA (e.g., which require 545 and 510 observations, respectively,
in the quadrotor scenario).
\begin{figure}[tb!]
\vspace{-.25em}
\centering
\includegraphics[width=.28\textwidth,clip,trim=0.2in 0.05in 0.5in 0.0in]{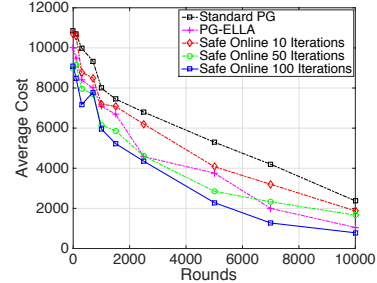}
\vspace{-1em}
\caption{Performance on quadrotor control.}
\label{fig:Quad}
\vspace{-1.5em}
\end{figure}


\begin{figure}[tb!]
\vspace{-0.5em}
\centering
\includegraphics[width=.35\textwidth,clip,trim=0.2in 1.25in 0.65in 1.5in]{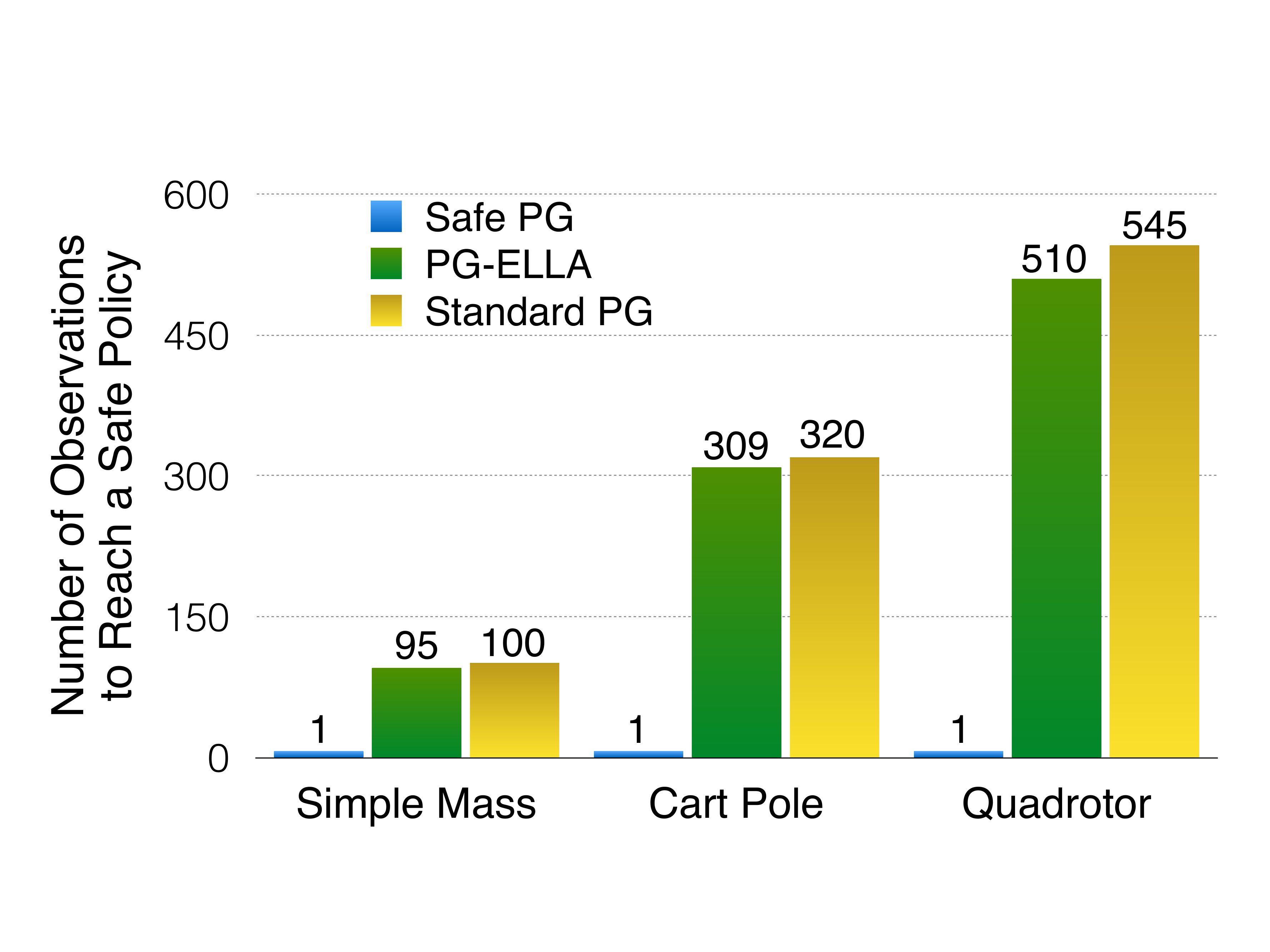}
\vspace{-1em}
\caption{Average number of task observations before acquiring policy parameters that abide by the constraints, showing that our approach immediately projects policies to safe regions.}
\label{Fig:NumObs}
\vspace{-.5em}
\end{figure}

\section{Conclusion}

We described the first lifelong PG learner that provides sublinear regret $\mathcal{O}(\sqrt{\vNumTotalRounds})$ with $\vNumTotalRounds$ total rounds.  In addition, our approach supports safety constraints on the learned policy, which are essential for robust learning in real applications. Our framework formalizes
lifelong learning as online MTL with limited
resources, and enables safe transfer by sharing policy parameters through
a latent knowledge base that is efficiently updated over time.



\newpage{}

\onecolumn

%
%
\setlength{\parskip}{5pt  minus .75pt}

\appendix
%

\section{Update Equations Derivation}\label{App:Update}
In this appendix, we derive the update equations for $\bm{L}$ and $\bm{S}$ in the special case of Gaussian policies. Please note that these derivations can be easily extended to other policy forms in higher dimensional action spaces. 

For a task $t_{j}$, the policy $\pi^{(t_{j})}_{\bm{\alpha}_{t_{j}}}\left(\bm{u}_{m}^{(k,t_{j})}|\bm{x}_{m}^{(k,t_{j})}\right)$ is given by: 
\begin{equation*}
\pi^{(t_{j})}_{\bm{\alpha}_{t_{j}}}\left(\bm{u}_{m}^{(k,t_{j})}|\bm{x}_{m}^{(k,t_{j})}\right)= \frac{1}{\sqrt{2\pi\sigma_{t_{j}}^{2}}}\exp\left(-\frac{1}{2\sigma_{t_{j}}^{2}}\left(\bm{u}_{m}^{(k,t_{j})}-\left(\bm{L}\bm{s}_{t_{j}}\right)^{\mathsf{T}}\bm{\Phi}\left(\bm{x}_{m}^{(k,t_{j})}\right)\right)^{2}\right). 
\end{equation*}
Therefore, the safe lifelong reinforcement learning optimization objective can be written as: 
\begin{equation}\label{Eq:LifelongApp}
\bm{e}_{r}(\bm{L},\bm{S})= \sum_{j=1}^{r} \frac{\eta_{t_{j}}}{2\sigma_{t_{j}}^{2}n_{t_{j}}} \sum_{k=1}^{n_{t_{j}}}\sum_{m=0}^{M_{t_{j}}-1} \left(\bm{u}_{m}^{(k,t_{j})}-\left(\bm{L}\bm{s}_{t_{j}}\right)^{\mathsf{T}}\bm{\Phi}\left(\bm{x}_{m}^{(k,t_{j})}\right)\right)^{2} + \mu_{1} ||\bm{S}||_{\mathsf{F}}^{2} + \mu_{2} ||\bm{L}||_{\mathsf{F}}^{2} \enspace .
\end{equation}
To arrive at the update equations, we need to derive Eq.~\eqref{Eq:LifelongApp} with respect to each $\bm{L}$ and $\bm{S}$. 
\subsection{Update Equations for $\bm{L}$}
Starting with the derivative of $\bm{e}_{r}(\bm{L},\bm{S})$ with respect to the shared repository $\bm{L}$, we can write: 
\begin{align*}
\nabla_{\bm{L}} \bm{e}_{r}(\bm{L},\bm{S}) &= \nabla_{\bm{L}}\left[\sum_{j=1}^{r} \frac{\eta_{t_{j}}}{2\sigma_{t_{j}}^{2}n_{t_{j}}} \sum_{k=1}^{n_{t_{j}}}\sum_{m=0}^{M_{t_{j}}-1} \left(\bm{u}_{m}^{(k,t_{j})}-\left(\bm{L}\bm{s}_{t_{j}}\right)^{\mathsf{T}}\bm{\Phi}\left(\bm{x}_{m}^{(k,t_{j})}\right)\right)^{2} + \mu_{1} ||\bm{S}||_{\mathsf{F}}^{2} + \mu_{2} ||\bm{L}||_{\mathsf{F}}^{2}
\right] \\
& = - \sum_{j=1}^{r}\left[ \frac{\eta_{t_{j}}}{\sigma_{t_{j}}^{2}n_{t_{j}}} \sum_{k=1}^{n_{t_{j}}}\sum_{m=0}^{M_{t_{j}}-1} \left(\bm{u}_{m}^{(k,t_{j})}-\left(\bm{L}\bm{s}_{t_{j}}\right)^{\mathsf{T}}\bm{\Phi}\left(\bm{x}_{m}^{(k,t_{j})}\right)\right) \bm{\Phi}\left(\bm{x}_{m}^{(k,t_{j})}\right)\bm{s}_{t_{j}}^{\mathsf{T}}\right] + 2 \mu_{2}\bm{L} \enspace .
\end{align*}
To acquire the minimum, we set the above to zero: 
\begin{align*}
&\sum_{j=1}^{r}\left[ \frac{\eta_{t_{j}}}{\sigma_{t_{j}}^{2}n_{t_{j}}} \sum_{k=1}^{n_{t_{j}}}\sum_{m=0}^{M_{t_{j}}-1} \left(\bm{u}_{m}^{(k,t_{j})}-\left(\bm{L}\bm{s}_{t_{j}}\right)^{\mathsf{T}}\bm{\Phi}\left(\bm{x}_{m}^{(k,t_{j})}\right)\right) \bm{\Phi}\left(\bm{x}_{m}^{(k,t_{j})}\right)\bm{s}_{t_{j}}^{\mathsf{T}}\right] + 2 \mu_{2}\bm{L}  = 0 \\
&\sum_{j=1}^{r} \left[\frac{\eta_{t_{j}}}{\sigma_{t_{j}}^{2} n_{t_{j}}} \sum_{k=1}^{n_{t_{j}}} \sum_{m=0}^{M_{t_{j}}-1} \bm{s}_{t_{j}}^{\mathsf{T}}\bm{L}^{\mathsf{T}} \bm{\Phi}\left(\bm{x}_{m}^{(k,t_{j})}\right) \bm{\Phi}\left(\bm{x}_{m}^{(k,t_{j})}\right) \bm{s}_{t_{j}}^{\mathsf{T}} \right] + 2\mu_{2} \bm{L} = \sum_{j=1}^{r} \frac{\eta_{t_{j}}}{\sigma_{t_{j}}^{2}n_{t_{j}}} \sum_{k=1}^{n_{t_{j}}}\sum_{m=0}^{M_{t_{j}}-1} \bm{u}_{m}^{(k,t_{j})} \bm{\Phi}\left(\bm{x}_{m}^{(k,t_{j})}\right) \bm{s}_{t_{j}}^{\mathsf{T}} \enspace .
\end{align*}
Noting that $\bm{s}_{t_{j}}^{\mathsf{T}}\bm{L}^{\mathsf{T}} \bm{\Phi}\left(\bm{x}_{m}^{(k,t_{j})}\right) \in \mathbb{R}$, we can write: 
\begin{equation}\label{Eq:LDer}
\sum_{j=1}^{r} \left[\frac{\eta_{t_{j}}}{\sigma_{t_{j}}^{2} n_{t_{j}}} \sum_{k=1}^{n_{t_{j}}} \sum_{m=0}^{M_{t_{j}}-1} \bm{\Phi}\left(\bm{x}_{m}^{(k,t_{j})}\right) \bm{s}_{t_{j}}^{\mathsf{T}} \bm{\Phi}^{\mathsf{T}}\left(\bm{x}_{m}^{(k,t_{j})}\right) \bm{L}\bm{s}_{t_{j}} \right] + 2\mu_{2} \bm{L} = \sum_{j=1}^{r} \frac{\eta_{t_{j}}}{\sigma_{t_{j}}^{2}n_{t_{j}}} \sum_{k=1}^{n_{t_{j}}}\sum_{m=0}^{M_{t_{j}}-1} \bm{u}_{m}^{(k,t_{j})} \bm{\Phi}\left(\bm{x}_{m}^{(k,t_{j})}\right) \bm{s}_{t_{j}}^{\mathsf{T}} \enspace . 
\end{equation}
To solve Eq.~\eqref{Eq:LDer}, we introduce the standard $\text{vec}(\cdot)$ operator leading to:
\begin{align*}
& \text{vec}\left(\sum_{j=1}^{r} \left[\frac{\eta_{t_{j}}}{\sigma_{t_{j}}^{2} n_{t_{j}}} \sum_{k=1}^{n_{t_{j}}} \sum_{m=0}^{M_{t_{j}}-1} \bm{\Phi}\left(\bm{x}_{m}^{(k,t_{j})}\right) \bm{s}_{t_{j}}^{\mathsf{T}} \bm{\Phi}^{\mathsf{T}}\left(\bm{x}_{m}^{(k,t_{j})}\right) \bm{L}\bm{s}_{t_{j}} \right] + 2\mu_{2} \bm{L}\right)  \\
&\hspace{26em} = \text{vec}\left(  \sum_{j=1}^{r} \frac{\eta_{t_{j}}}{\sigma_{t_{j}}^{2}n_{t_{j}}} \sum_{k=1}^{n_{t_{j}}}\sum_{m=0}^{M_{t_{j}}-1} \bm{u}_{m}^{(k,t_{j})} \bm{\Phi}\left(\bm{x}_{m}^{(k,t_{j})}\right) \bm{s}_{t_{j}}^{\mathsf{T}} \right) \\
&\sum_{j=1}^{r} \frac{\eta_{t_{j}}}{\sigma_{t_{j}}^{2} n_{t_{j}}} \sum_{k=1}^{n_{t_{j}}} \sum_{m=0}^{M_{t_{j}}-1} \text{vec}\left({\bm{\Phi}\left(\bm{x}_{m}^{(k,t_{j})}\right)} \bm{s}_{t_{j}}^{\mathsf{T}}\right) \text{vec}\left(\bm{\Phi}^{\mathsf{T}}\left(\bm{x}_{m}^{(k,t_{j})}\right) \bm{L}\bm{s}_{t_{j}}\right) + 2\mu_{2} \text{vec}(\bm{L}) \\
&\hspace{26em}= \sum_{j=1}^{r} \frac{\eta_{t_{j}}}{\sigma_{t_{j}}^{2}n_{t_{j}}} \sum_{k=1}^{n_{t_{j}}}\sum_{m=0}^{M_{t_{j}}-1} \text{vec}\left(\bm{u}_{m}^{(k,t_{j})} \bm{\Phi}\left(\bm{x}_{m}^{(k,t_{j})}\right) \bm{s}_{t_{j}}^{\mathsf{T}} \right) \enspace .
\end{align*}
Knowing that for a given set of matrices $\bm{A}$, $\bm{B}$, and $\bm{X}$,  $\text{vec}(\bm{A}\bm{X}\bm{B}) = \left(\bm{B}^{\mathsf{T}}\otimes \bm{A} \right)\text{vec}(\bm{X})$, we can write 
\begin{align*}
&\sum_{j=1}^{r} \frac{\eta_{t_{j}}}{\sigma_{t_{j}}^{2} n_{t_{j}}} \sum_{k=1}^{n_{t_{j}}} \sum_{m=0}^{M_{t_{j}}-1} \text{vec}\left({\bm{\Phi}\left(\bm{x}_{m}^{(k,t_{j})}\right)} \bm{s}_{t_{j}}^{\mathsf{T}}\right) \left(\bm{s}_{t_{j}}^{\mathsf{T}} \otimes \bm{\Phi}^{\mathsf{T}}\left(\bm{x}_{m}^{(k,t_{j})}\right)\right)\text{vec}(\bm{L}) + 2\mu_{2} \text{vec}(\bm{L}) \\
&\hspace{26em}= \sum_{j=1}^{r} \frac{\eta_{t_{j}}}{\sigma_{t_{j}}^{2}n_{t_{j}}} \sum_{k=1}^{n_{t_{j}}}\sum_{m=0}^{M_{t_{j}}-1} \text{vec}\left(\bm{u}_{m}^{(k,t_{j})} \bm{\Phi}\left(\bm{x}_{m}^{(k,t_{j})}\right) \bm{s}_{t_{j}}^{\mathsf{T}} \right) \enspace .
\end{align*} 
By choosing $\bm{Z}_{\bm{L}}= 2\mu_{2} \bm{I}_{dk\times dk} +\sum_{j=1}^{r}\frac{\eta_{t_{j}}}{n_{t_{j}}\sigma_{t_{j}}^{2}}\sum_{k=1}^{n_{t_{j}}}\sum_{m=0}^{M_{t_{j}}-1} \text{vec}\left(\bm{\Phi}\left(\bm{x}_{m}^{(k,t_{j})}\right)\bm{s}_{t_{j}}^{\mathsf{T}}\right)\left(\bm{\Phi}\left(\bm{x}_{m}^{(k,t_{j})}\right) \otimes \bm{s}_{t_{j}}^{\mathsf{T}}\right)$, and $\bm{v}_{\bm{L}}=\sum_{j=1}^{r} \frac{\eta_{t_{j}}}{n_{t_{j}}\sigma_{t_{j}}^{2}} \sum_{k=1}^{n_{t_{j}}}\sum_{m=0}^{M_{t_{j}}-1}\text{vec}\left(\bm{u}_{m}^{(k,t_{j})}\bm{\Phi}\left(\bm{x}_{m}^{(k,t_{j})}\right)\bm{s}_{t_{j}}^{\mathsf{T}}\right)$, we can update $\bm{L}=\bm{Z}_{\bm{L}}^{-1}\bm{v}_{\bm{L}}$. 

\subsection{Update Equations for $\bm{S}$}
To derive the update equations with respect to $\bm{S}$,  similar approach to that of $\bm{L}$ can be followed. The derivative of $\bm{e}_{r}(\bm{L},\bm{S})$ with respect to $\bm{S}$ can be computed column-wise for all tasks observed so far: 
\begin{align*}
\nabla_{\bm{s}_{t_{j}}} \bm{e}_{r}(\bm{L},\bm{S}) &= \nabla_{\bm{s}_{t_{j}}}\left[\sum_{j=1}^{r} \frac{\eta_{t_{j}}}{2\sigma_{t_{j}}^{2}n_{t_{j}}} \sum_{k=1}^{n_{t_{j}}}\sum_{m=0}^{M_{t_{j}}-1} \left(\bm{u}_{m}^{(k,t_{j})}-\left(\bm{L}\bm{s}_{t_{j}}\right)^{\mathsf{T}}\bm{\Phi}\left(\bm{x}_{m}^{(k,t_{j})}\right)\right)^{2} + \mu_{1} ||\bm{S}||_{\mathsf{F}}^{2} + \mu_{2} ||\bm{L}||_{\mathsf{F}}^{2}
\right] \\
& = - \sum_{t_{k}=t_{j}}\left[ \frac{\eta_{t_{j}}}{\sigma_{t_{j}}^{2}n_{t_{j}}} \sum_{k=1}^{n_{t_{j}}}\sum_{m=0}^{M_{t_{j}}-1} \left(\bm{u}_{m}^{(k,t_{j})}-\left(\bm{L}\bm{s}_{t_{j}}\right)^{\mathsf{T}}\bm{\Phi}\left(\bm{x}_{m}^{(k,t_{j})}\right)\right) \bm{L}^{\mathsf{T}}\bm{\Phi}\left(\bm{x}_{m}^{(k,t_{j})}\right)\right] + 2 \mu_{2}\bm{s}_{t_{j}} \enspace .
\end{align*}
Using a similar analysis to the previous section, choosing 
\begin{align*}
\bm{Z}_{\bm{s}_{t_{j}}} &= 2\mu_{1} \bm{I}_{k \times k} + \sum_{t_{k}=t_{j}}\frac{\eta_{t_{j}}}{n_{t_{j}}\sigma_{t_{j}}^{2}}\sum_{k=1}^{n_{t_{j}}}\sum_{m=0}^{M_{t_{j}}-1}\bm{L}^{\mathsf{T}}\bm{\Phi}\left(\bm{x}_{m}^{(k,t_{j})}\right)\bm{\Phi}^{\mathsf{T}}\left(\bm{x}_{m}^{(k,t_{j})}\right)\bm{L} \enspace ,\\
\bm{v}_{\bm{s}_{t_{j}}} & =\sum_{t_{k}=t_{j}}\frac{\eta_{t_{j}}}{n_{t_{j}}\sigma_{t_{j}}^{2}}\sum_{k=1}^{n_{t_{j}}}\sum_{m=0}^{M_{t_{j}}-1}\bm{u}_{m}^{(k,t_{j})}\bm{L}^{\mathsf{T}}\bm{\Phi}\left(\bm{x}_{m}^{(k,t_{j})}\right) \enspace ,
\end{align*}
we can update $\bm{s}_{t_{j}}=\bm{Z}_{s_{t_{j}}}^{-1}\bm{v}_{\bm{s}_{t_{j}}}$.

\section{Proofs of Theoretical Guarantees} \label{appendix:proofs}

In this appendix, we prove the claims and lemmas from the main paper, leading to sublinear regret (Theorem 1).






\begin{mylemma1}
Assume the policy for a task $\vTaskJ$ at a round $\vCurrentRound$
to be given by $\pi_{\bm{\alpha}_{\vTaskJ}}^{\left(\vTaskJ\right)}\left(\bm{u}_{m}^{\left(k,\ \vTaskJ\right)}|\bm{x}_{m}^{\left(k,\ \vTaskJ\right)}\right)\Big|_{\hat{\bm{\theta}}_{\vCurrentRound}}=\mathcal{N}\left(\bm{\alpha}_{\vTaskJ}^{\mathsf{T}}\Big|_{\hat{\bm{\theta}}_{\vCurrentRound}}\bm{\Phi}\left(\bm{x}_{m}^{\left(k,\ \vTaskJ\right)}\right),\bm{\sigma}_{\vTaskJ}\right)$,
for $\bm{x}_{m}^{\left(k,\ \vTaskJ\right)}\in\mathcal{X}_{\vTaskJ}$
and $\bm{u}_{m}^{\left(k,\ \vTaskJ\right)}\in\mathcal{U}_{\vTaskJ}$
with $\mathcal{X}_{\vTaskJ}$ and $\mathcal{U}_{\vTaskJ}$ representing
the state and action spaces, respectively. The gradient $\nabla_{\bm{\alpha}_{\vTaskJ}}l_{\vTaskJ}\left(\bm{\alpha}_{\vTaskJ}\right)\Big|_{\hat{\bm{\theta}}_{\vCurrentRound}}$,
for $l_{\vTaskJ}\left(\bm{\alpha}_{\vTaskJ}\right)=-\sfrac{1}{n_{\vTaskJ}}\sum_{k=1}^{n_{\vTaskJ}}\sum_{m=0}^{M_{\vTaskJ}-1}\log\left[\pi_{\bm{\alpha}_{\vTaskJ}}^{\left(\vTaskJ\right)}\left(\bm{u}_{m}^{\left(k,\ \vTaskJ\right)}|\bm{x}_{m}^{\left(k,\ \vTaskJ\right)}\right)\right]$
satisfies 
\[
\left|\left|\nabla_{\bm{\alpha}_{\vTaskJ}}l_{\vTaskJ}\left(\bm{\alpha}_{\vTaskJ}\right)\Big|_{\hat{\bm{\theta}}_{\vCurrentRound}}\right|\right|_{2}\leq\frac{M_{\vTaskJ}}{\sigma_{\vTaskJ}^{2}}\left[\left(u_{\text{max}}+\max_{\vTaskK\in\mathcal{I}_{\vCurrentRound-1}}\left\{ \left|\left|\bm{A}_{\vTaskK}^{+}\right|\right|_{2}\left(\left|\left|\bm{b}_{\vTaskK}\right|\right|_{2}+\bm{c}_{\text{max}}\right)\right\} \bm{\Phi}_{\text{max}}\right)\bm{\Phi}_{\text{max}}\right],
\]
with $u_{\text{max}}=\max_{k,m}\left\{ \left|\bm{u}_{m}^{\left(k,\ \vTaskJ\right)}\right|\right\} $
and $\bm{\Phi}_{\text{max}}=\max_{k,m}\left\{ \left|\left|\bm{\Phi}\left(\bm{x}_{m}^{\left(k,\ \vTaskJ\right)}\right)\right|\right|_{2}\right\} $
for all trajectories and all tasks. \end{mylemma1}

\begin{proof} The proof of the above lemma will be provided as a
collection of claims. We start with the following: \begin{claim}
Given $\pi_{\bm{\alpha}_{\vTaskJ}}^{\left(\vTaskJ\right)}\left(\bm{u}_{m}^{(k)}|\bm{x}_{m}^{(k)}\right)\Big|_{\hat{\bm{\theta}}_{\vCurrentRound}}=\mathcal{N}\left(\bm{\alpha}_{\vTaskJ}^{\mathsf{T}}\Big|_{\hat{\bm{\theta}}_{\vCurrentRound}}\bm{\Phi}\left(\bm{x}_{m}^{\left(k,\ \vTaskJ\right)}\right),\bm{\sigma}_{\vTaskJ}\right)$,
for $\bm{x}_{m}^{\left(k,\ \vTaskJ\right)}\in\mathcal{X}_{\vTaskJ}$
and $\bm{u}_{m}^{\left(k,\ \vTaskJ\right)}\in\mathcal{U}_{\vTaskJ}$,
and $l_{\vTaskJ}\left(\bm{\alpha}_{\vTaskJ}\right)=-\sfrac{1}{n_{\vTaskJ}}\sum_{k=1}^{n_{\vTaskJ}}\sum_{m=0}^{M_{\vTaskJ}-1}\log\left[\pi_{\bm{\alpha}_{\vTaskJ}}^{\left(\vTaskJ\right)}\left(\bm{u}_{m}^{\left(k,\ \vTaskJ\right)}|\bm{x}_{m}^{\left(k,\ \vTaskJ\right)}\right)\right]$,
$\left|\left|\nabla_{\bm{\alpha}_{\vTaskJ}}l_{\vTaskJ}\left(\bm{\alpha}_{\vTaskJ}\right)\Big|_{\hat{\bm{\theta}}_{\vCurrentRound}}\right|\right|_{2}$
satisfies 
\begin{equation}
\left|\left|\nabla_{\bm{\alpha}_{\vTaskJ}}l_{\vTaskJ}\left(\bm{\alpha}_{\vTaskJ}\right)\Big|_{\hat{\bm{\theta}}_{\vCurrentRound}}\right|\right|_{2}\leq\frac{M_{\vTaskJ}}{\sigma_{\vTaskJ}^{2}}\left[\left(u_{\text{max}}+\left|\left|\bm{\alpha}_{\vTaskJ}\Big|_{\hat{\bm{\theta}}_{\vCurrentRound}}\right|\right|_{2}\bm{\Phi}_{\text{max}}\right)\bm{\Phi}_{\text{max}}\right].\label{Eq:ClaimOne}
\end{equation}

\end{claim}

\begin{claimproof} Since $\pi_{\bm{\alpha}_{\vTaskJ}}^{\left(\vTaskJ\right)}\left(\bm{u}_{m}^{\left(k,\ \vTaskJ\right)}|\bm{x}_{m}^{\left(k,\ \vTaskJ\right)}\right)\Big|_{\hat{\bm{\theta}}_{\vCurrentRound}}=\mathcal{N}\left(\bm{\alpha}_{\vTaskJ}^{\mathsf{T}}\Big|_{\hat{\bm{\theta}}_{\vCurrentRound}}\bm{\Phi}\left(\bm{x}_{m}^{\left(k,\ \vTaskJ\right)}\right),\bm{\sigma}_{\vTaskJ}\right)$,
we can write 
\[
\log\left[\pi_{\bm{\alpha}_{\vTaskJ}}^{\left(\vTaskJ\right)}\left(\bm{u}_{m}^{\left(k,\ \vTaskJ\right)}|\bm{x}_{m}^{\left(k,\ \vTaskJ\right)}\right)\Big|_{\hat{\bm{\theta}}_{\vCurrentRound}}\right]=-\log\left[\sqrt{2\pi\sigma_{\vTaskJ}^{2}}\right]-\frac{1}{2\sigma_{\vTaskJ}^{2}}\left(\bm{u}_{m}^{\left(k,\ \vTaskJ\right)}-\bm{\alpha}_{\vTaskJ}^{\mathsf{T}}\Big|_{\hat{\bm{\theta}}_{\vCurrentRound}}\bm{\Phi}\left(\bm{x}_{m}^{\left(k,\ \vTaskJ\right)}\right)\right)^{2}.
\]
Therefore: 
\begin{align*}
\nabla_{\bm{\alpha}_{\vTaskJ}}l_{\vTaskJ}\left(\bm{\alpha}_{\vTaskJ}\right)\Big|_{\hat{\bm{\theta}}_{\vCurrentRound}} & =-\frac{1}{n_{\vTaskJ}}\sum_{k=1}^{n_{\vTaskJ}}\sum_{m=0}^{M_{\vTaskJ}-1}\frac{1}{\sigma_{\vTaskJ}^{2}}\left(\bm{u}_{m}^{\left(k,\ \vTaskJ\right)}-\bm{\alpha}_{\vTaskJ}^{\mathsf{T}}\Big|_{\hat{\bm{\theta}}_{\vCurrentRound}}\bm{\Phi}\left(\bm{x}_{m}^{\left(k,\ \vTaskJ\right)}\right)\right)\bm{\Phi}\left(\bm{x}_{m}^{\left(k,\ \vTaskJ\right)}\right)\\
\left|\left|\nabla_{\bm{\alpha}_{\vTaskJ}}l_{\vTaskJ}\left(\bm{\alpha}_{\vTaskJ}\right)\Big|_{\hat{\bm{\theta}}_{\vCurrentRound}}\right|\right|_{2} & \leq\frac{M_{\vTaskJ}}{\sigma_{\vTaskJ}^{2}}\left[\max_{k,m}\left\{ \left|\bm{u}_{m}^{\left(k,\ \vTaskJ\right)}-\bm{\alpha}_{\vTaskJ}^{\mathsf{T}}\Big|_{\hat{\bm{\theta}}_{\vCurrentRound}}\bm{\Phi}\left(\bm{x}_{m}^{\left(k,\ \vTaskJ\right)}\right)\right|\times\left|\left|\bm{\Phi}\left(\bm{x}_{m}^{\left(k,\ \vTaskJ\right)}\right)\right|\right|_{2}\right\} \right]\\
 & \leq\frac{M_{\vTaskJ}}{\sigma_{\vTaskJ}^{2}}\Bigg[\max_{k,m}\left\{ \left|\bm{u}_{m}^{\left(k,\ \vTaskJ\right)}\right|\times\left|\left|\bm{\Phi}\left(\bm{x}_{m}^{\left(k,\ \vTaskJ\right)}\right)\right|\right|_{2}\right\} \\
 & \hspace{12em}+\max_{k,m}\left\{ \left|\bm{\alpha}_{\vTaskJ}^{\mathsf{T}}\Big|_{\hat{\bm{\theta}}_{\vCurrentRound}}\bm{\Phi}\left(\bm{x}_{m}^{\left(k,\ \vTaskJ\right)}\right)\right|\times\left|\left|\bm{\Phi}\left(\bm{x}_{m}^{\left(k,\ \vTaskJ\right)}\right)\right|\right|_{2}\right\} \Bigg]\\
 & \leq\frac{M_{\vTaskJ}}{\sigma_{\vTaskJ}^{2}}\Bigg[\max_{k,m}\left\{ \left|\bm{u}_{m}^{\left(k,\ \vTaskJ\right)}\right|\right\} \max_{k,m}\left\{ \left|\left|\bm{\Phi}\left(\bm{x}_{m}^{\left(k,\ \vTaskJ\right)}\right)\right|\right|_{2}\right\} \\
 & \hspace{8.1em}+\max_{k,m}\left\{ \left|\left\langle \bm{\alpha}_{\vTaskJ}\Big|_{\hat{\bm{\theta}}_{\vCurrentRound}},\bm{\Phi}\left(\bm{x}_{m}^{\left(k,\ \vTaskJ\right)}\right)\right\rangle \right|\right\} \max_{k,m}\left\{ \left|\left|\bm{\Phi}\left(\bm{x}_{m}^{\left(k,\ \vTaskJ\right)}\right)\right|\right|_{2}\right\} \Bigg] \enspace .
\end{align*}
Denoting $\max_{k,m}\left\{ \left|\bm{u}_{m}^{\left(k,\ \vTaskJ\right)}\right|\right\} =u_{\text{max}}$
and $\max_{k,m}\left\{ \left|\left|\bm{\Phi}\left(\bm{x}_{m}^{\left(k,\ \vTaskJ\right)}\right)\right|\right|_{2}\right\} =\bm{\Phi}_{\text{max}}$
for all trajectories and all tasks, we can write 
\[
\left|\left|\nabla_{\bm{\alpha}_{\vTaskJ}}l_{\vTaskJ}\left(\bm{\alpha}_{\vTaskJ}\right)\Big|_{\hat{\bm{\theta}}_{\vCurrentRound}}\right|\right|_{2}\leq\frac{M_{\vTaskJ}}{\sigma_{\vTaskJ}^{2}}\left[\left(u_{\text{max}}+\max_{k,m}\left\{ \left|\left\langle \bm{\alpha}_{\vTaskJ}\Big|_{\hat{\bm{\theta}}_{\vCurrentRound}},\bm{\Phi}\left(\bm{x}_{m}^{\left(k,\ \vTaskJ\right)}\right)\right\rangle \right|\right\} \right)\bm{\Phi}_{\text{max}}\right] \enspace .
\]
Using the Cauchy-Shwarz inequality~\cite{Horn199063}, we can upper
bound $\max_{k,m}\left\{ \left|\left\langle \bm{\alpha}_{\vTaskJ}\Big|_{\hat{\bm{\theta}}_{\vCurrentRound}},\bm{\Phi}\left(\bm{x}_{m}^{\left(k,\ \vTaskJ\right)}\right)\right\rangle \right|\right\} $
as 
\begin{align*}
\max_{k,m}\left\{ \left|\left\langle \bm{\alpha}_{\vTaskJ}\Big|_{\hat{\bm{\theta}}_{\vCurrentRound}},\bm{\Phi}\left(\bm{x}_{m}^{\left(k,\ \vTaskJ\right)}\right)\right\rangle \right|\right\} \leq\max_{k,m}\left\{ \left|\left|\bm{\alpha}_{\vTaskJ}\Big|_{\hat{\bm{\theta}}_{\vCurrentRound}}\right|\right|_{2}\left|\left|\bm{\Phi}\left(\bm{x}_{m}^{\left(k,\ \vTaskJ\right)}\right)\right|\right|_{2}\right\}  & \leq\max_{k,m}\left\{ \left|\left|\bm{\alpha}_{\vTaskJ}\Big|_{\hat{\bm{\theta}}_{\vCurrentRound}}\right|\right|_{2}\right\} \bm{\Phi}_{\text{max}}\\
 & \leq\left|\left|\bm{\alpha}_{\vTaskJ}\Big|_{\hat{\bm{\theta}}_{\vCurrentRound}}\right|\right|_{2}\bm{\Phi}_{\text{max}} \enspace . 
\end{align*}
Finalizing the statement of the claim, the overall bound on the norm
of the gradient of $l_{\vTaskJ}(\bm{\alpha}_{\vTaskJ})$ can be
written as 
\begin{equation}
\left|\left|\nabla_{\bm{\alpha}_{\vTaskJ}}l_{\vTaskJ}\left(\bm{\alpha}_{\vTaskJ}\right)\Big|_{\hat{\bm{\theta}}_{\vCurrentRound}}\right|\right|_{2}\leq\frac{M_{\vTaskJ}}{\sigma_{\vTaskJ}^{2}}\left[\left(u_{\text{max}}+\left|\left|\bm{\alpha}_{\vTaskJ}\Big|_{\hat{\bm{\theta}}_{\vCurrentRound}}\right|\right|_{2}\bm{\Phi}_{\text{max}}\right)\bm{\Phi}_{\text{max}}\right] \enspace .\label{Eq:One}
\end{equation}
\end{claimproof}

\begin{claim}
The norm of the gradient of the loss function satisfies: 
\[
\left|\left|\nabla_{\bm{\alpha}_{\vTaskJ}}l_{\vTaskJ}\left(\bm{\alpha}_{\vTaskJ}\right)\Big|_{\hat{\bm{\theta}}_{\vCurrentRound}}\right|\right|_{2}\leq\frac{M_{\vTaskJ}}{\sigma_{\vTaskJ}^{2}}\left[\left(u_{\text{max}}+\max_{\vTaskK\in\mathcal{I}_{\vCurrentRound-1}}\left\{ \|\bm{A}_{\vTaskK}^{+}\|_{2}\left(\|\bm{b}_{\vTaskK}\|_{2}+\bm{c}_{\text{max}}\right)\right\} \bm{\Phi}_{\text{max}}\right)\bm{\Phi}_{\text{max}}\right] \enspace . 
\]
\end{claim} \begin{claimproof} As mentioned previously, we consider
the linearization of the loss function $l_{\vTaskJ}$ around the
constraint solution of the previous round, $\hat{\bm{\theta}}_{\vCurrentRound}$.
Since $\hat{\bm{\theta}}_{\vCurrentRound}$ satisfies $\bm{A}_{\vTaskK}\bm{\alpha}_{\vTaskK}=\bm{b}_{\vTaskK}-\bm{c}_{\vTaskK},\forall \vTaskK\in\mathcal{I}_{\vCurrentRound-1}$.
Hence, we can write 
\begin{align*}
\bm{A}_{\vTaskK}\bm{\bm{\alpha}_{\vTaskK}}+\bm{c}_{\vTaskK} & =\bm{b}_{\vTaskK}\ \ \ \ \forall \vTaskK\in\mathcal{I}_{\vCurrentRound-1}\\
\implies\bm{\alpha}_{\vTaskK} & =\bm{A}_{\vTaskK}^{+}\left(\bm{b}_{\vTaskK}-\bm{c}_{\vTaskK}\right)\ \ \ \text{with \ensuremath{\bm{A}_{\vTaskK}^{+}=\left(\bm{A}_{\vTaskK}^{\mathsf{T}}\bm{A}_{\vTaskK}\right)^{-1}\bm{A}_{\vTaskK}^{\mathsf{T}}} being the left pseudo-inverse.}
\end{align*}
Therefore 
\begin{align*}
\left|\left|\bm{\alpha}_{\vTaskK}\right|\right|_{2} & \leq\left|\left|\bm{A}_{\vTaskK}^{+}\right|\right|_{2}\left(\left|\left|\bm{b}_{\vTaskK}\right|\right|_{2}+\left|\left|\bm{c}_{\vTaskK}\right|\right|_{2}\right)\\
 & \leq\left|\left|\bm{A}_{\vTaskK}^{+}\right|\right|_{2}\left(\left|\left|\bm{b}_{\vTaskK}\right|\right|_{2}+\bm{c}_{\text{max}}\right) \enspace .
\end{align*}

Combining the above results with those of Eq.~\eqref{Eq:One}
we arrive at 
\[
\left|\left|\nabla_{\bm{\alpha}_{\vTaskJ}}l_{\vTaskJ}\left(\bm{\alpha}_{\vTaskJ}\right)\right|\right|_{2}\leq\frac{M_{\vTaskJ}}{\sigma_{\vTaskJ}^{2}}\left[\left(u_{\text{max}}+\max_{\vTaskK\in\mathcal{I}_{\vCurrentRound-1}}\left\{ \left|\left|\bm{A}_{\vTaskK}^{+}\right|\right|_{2}\left(\left|\left|\bm{b}_{\vTaskK}\right|\right|_{2}+\bm{c}_{\text{max}}\right)\right\} \bm{\Phi}_{\text{max}}\right)\bm{\Phi}_{\text{max}}\right] \enspace .
\] 
\end{claimproof}

The previous result finalizes the statement of the lemma, bounding
the gradient of the loss function in terms of the \emph{safety} constraints.

\end{proof}

\begin{mylemma2}
The norm of the gradient
of the loss function evaluated at $\hat{\bm{\theta}}_{\vCurrentRound}$
satisfies 
\begin{align*}
\left|\left|\nabla_{\bm{\theta}}l_{\vTaskJ}\left(\bm{\theta}\right)\Big|_{\hat{\bm{\theta}}_{\vCurrentRound}}\right|\right|_{2}^{2} & \leq\Big|\Big|\nabla_{\bm{\alpha}_{\vTaskJ}}l_{\vTaskJ}\left(\bm{\theta}\right)\Big|_{\hat{\bm{\theta}}_{\vCurrentRound}}\Big|\Big|_{2}^{2}\Bigg(q\times d \left(\sfrac{2d}{p^{2}}\max_{\vTaskK\in\mathcal{I}_{\vCurrentRound-1}}\left\{ \left|\left|\bm{A}_{\vTaskJ}^{\dagger}\right|\right|_{2}^{2}\left(\left|\left|\bm{b}_{\vTaskJ}\right|\right|_{2}^{2}+\bm{c}_{\text{max}}^{2}\right)\right\} +1\right)\Bigg) \enspace .
\end{align*}
\end{mylemma2}

\begin{proof} The derivative of $l_{\vTaskJ}(\bm{\theta})\Big|_{\hat{\bm{\theta}}_{\vCurrentRound}}$ 
can be written as 
\begin{align*}
\nabla_{\bm{\theta}}l_{\vTaskJ}(\bm{\theta})\Big|_{\hat{\bm{\theta}}_{\vCurrentRound}} & =\left[\begin{array}{c}
\nabla_{\bm{\alpha}_{\vTaskJ}}l_{\vTaskJ}^{\mathsf{T}}(\bm{\theta})\Big|_{\hat{\bm{\theta}}_{\vCurrentRound}}\left[\begin{array}{c}
\frac{\partial\bm{\alpha}_{\vTaskJ}^{(1)}}{\partial\bm{\theta}_{1}}\Big|_{\hat{\bm{\theta}}_{\vCurrentRound}}\\
\vdots\\
\frac{\partial\bm{\alpha}_{\vTaskJ}^{(d)}}{\partial\bm{\theta}_{1}}\Big|_{\hat{\bm{\theta}}_{\vCurrentRound}}
\end{array}\right]\\
\vdots\\
\nabla_{\bm{\alpha}_{\vTaskJ}}l_{\vTaskJ}^{\mathsf{T}}(\bm{\theta})\Big|_{\hat{\bm{\theta}}_{\vCurrentRound}}\left[\begin{array}{c}
\frac{\partial\bm{\alpha}_{\vTaskJ}^{(1)}}{\partial\bm{\theta}_{dk+k\vNumTotalTasks}}\Big|_{\hat{\bm{\theta}}_{\vCurrentRound}}\\
\vdots\\
\frac{\partial\bm{\alpha}_{\vTaskJ}^{(d)}}{\partial\bm{\theta}_{dk+k\vNumTotalTasks}}\Big|_{\hat{\bm{\theta}}_{\vCurrentRound}}
\end{array}\right]
\end{array}\right]=\left[\begin{array}{c}
\nabla_{\bm{\alpha}_{\vTaskJ}}l_{\vTaskJ}^{\mathsf{T}}(\bm{\theta})\Big|_{\hat{\bm{\theta}}_{\vCurrentRound}}\left[\begin{array}{c}
\bm{\theta}_{dk+1}\Big|_{\hat{\bm{\theta}}_{\vCurrentRound}}\\
0\\
\vdots\\
0
\end{array}\right]\\
\vdots\\
\nabla_{\bm{\alpha}_{\vTaskJ}}l_{\vTaskJ}^{\mathsf{T}}(\bm{\theta})\Big|_{\hat{\bm{\theta}}_{\vCurrentRound}}\left[\begin{array}{c}
0\\
\vdots\\
\bm{\theta}_{(d+1)k+1}\Big|_{\hat{\bm{\theta}}_{\vCurrentRound}}
\end{array}\right]\\
\vdots\\
\nabla_{\bm{\alpha}_{\vTaskJ}}l_{\vTaskJ}^{\mathsf{T}}(\bm{\theta})\Big|_{\hat{\bm{\theta}}_{\vCurrentRound}}\left[\begin{array}{c}
\bm{\theta}_{d(k+1)+1}\Big|_{\hat{\bm{\theta}}_{\vCurrentRound}}\\
\vdots\\
\bm{\theta}_{dk}\Big|_{\hat{\bm{\theta}}_{\vCurrentRound}}
\end{array}\right]
\end{array}\right]\\
 & \implies \left\|\nabla_{\bm{\theta}}l_{\vTaskJ}(\bm{\theta})\Big|_{\hat{\bm{\theta}}_{\vCurrentRound}}\right\|_{2}^{2}\leq\left\|\nabla_{\bm{\alpha}_{\vTaskJ}}l_{\vTaskJ}(\bm{\alpha}_{\vTaskJ})\Big|_{\hat{\bm{\theta}}_{\vCurrentRound}}\right\|_{2}^{2}\left[d\left\|\bm{s}_{\vTaskJ}\Big|_{\hat{\bm{\theta}}_{\vCurrentRound}}\right\|_{2}^{2}+\left\|\bm{L}\Big|_{\hat{\bm{\theta}}_{\vCurrentRound}}\right\|_{\mathsf{F}}^{2}\right] \enspace .
\end{align*}

The results of Lemma~\ref{Lemma:Lemma1} bound $\left\|\nabla_{\bm{\alpha}_{\vTaskJ}}l_{\vTaskJ}(\bm{\theta})\Big|_{\hat{\bm{\theta}}_{\vCurrentRound}}\right\|_{2}^{2}$.

Now, we target to bound each of $\Big|\Big|\bm{s}_{\vTaskJ}\Big|_{\hat{\bm{\theta}}_{\vCurrentRound}}\Big|\Big|_{2}^{2}$
and $\Big|\Big|\bm{L}\Big|_{\hat{\bm{\theta}}_{\vCurrentRound}}\Big|\Big|_{\mathsf{F}}^{2}$.

\paragraph{Bounding $\left\|\bm{s}_{\vTaskJ}\Big|_{\hat{\bm{\theta}}_{\vCurrentRound}}\right\|_{2}^{2}$
and $\|\bm{L}\Big|_{\hat{\bm{\theta}}_{\vCurrentRound}}\|_{\mathsf{F}}^{2}$:}

Considering the constraint $\bm{A}_{\vTaskJ}\bm{L}\bm{s}_{\vTaskJ}+\bm{c}_{\vTaskJ}=\bm{b}_{\vTaskJ}$
for a task $\vTaskJ$, we realize that $\bm{s}_{{\vTaskJ}}=\bm{L}^{+}\left(\bm{A}_{\vTaskJ}^{+}\left(\bm{b}_{\vTaskJ}-\bm{c}_{\vTaskJ}\right)\right)$.
Therefore, 
\begin{align}
\left\|\bm{s}_{\vTaskJ}\Big|_{\hat{\theta}_{\vCurrentRound}}\right\|_{2}\leq\left|\left|\bm{L}^{+}\left(\bm{A}_{\vTaskJ}^{+}\left(\bm{b}_{\vTaskJ}-\bm{c}_{\vTaskJ}\right)\right)\right|\right|_{2} & \leq\left|\left|\bm{L}^{+}\right|\right|_{2}\left|\left|\bm{A}_{\vTaskJ}^{+}\right|\right|_{2}\left(\left|\left|\bm{b}_{\vTaskJ}\right|\right|_{2}+\left|\left|\bm{c}_{\vTaskJ}\right|\right|_{2}\right). \label{Eq:S}\\
\nonumber 
\end{align}
Noting that
\begin{align*}
\left|\left|\bm{L}^{+}\right|\right|_{\mathsf{2}}=\left|\left|\left(\bm{L}^{\mathsf{T}}\bm{L}\right)^{-1}\bm{L}^{\mathsf{T}}\right|\right|_{\mathsf{2}} & \leq\left|\left|\left(\bm{L}^{\mathsf{T}}\bm{L}\right)^{-1}\right|\right|_{2}\left|\left|\bm{L}^{\mathsf{T}}\right|\right|_{2}\leq\left|\left|\left(\bm{L}^{\mathsf{T}}\bm{L}\right)^{-1}\right|\right|_{2}\left|\left|\bm{L}^{\mathsf{T}}\right|\right|_{\mathsf{F}}\\
 & =\left|\left|\left(\bm{L}^{\mathsf{T}}\bm{L}\right)^{-1}\right|\right|_{2}\left|\left|\bm{L}\right|\right|_{\mathsf{F}} \enspace . 
\end{align*}
To relate $\left|\left|\bm{L}^{+}\right|\right|_{\mathsf{2}}$ to
$\left|\left|\bm{L}\right|\right|_{\mathsf{F}}$, we need to bound
$\left|\left|\left(\bm{L}^{\mathsf{T}}\bm{L}\right)^{-1}\right|\right|_{2}$
in terms of $\|\bm{L}\|_{\mathsf{F}}$. Denoting the spectrum of $\bm{L}^{\mathsf{T}}\bm{L}$
as $\text{spec}\left(\bm{L}^{\mathsf{T}}\bm{L}\right)=\left\{ \bm{\lambda}_{1},\dots,\bm{\lambda}_{k}\right\} $
such that $0<\bm{\lambda}_{1}\leq\dots\leq\bm{\lambda}_{k}$, then
$\text{spect}\left(\left(\bm{L}^{\mathsf{T}}\bm{L}\right)^{-1}\right)=\left\{ \sfrac{1}{\bm{\lambda}_{1}},\dots,\sfrac{1}{\bm{\lambda}_{k}}\right\} $
such that $\sfrac{1}{\bm{\lambda}_{k}}\leq\dots\leq\sfrac{1}{\bm{\lambda}_{k}}$.
Hence, $\left|\left|\left(\bm{L}^{\mathsf{T}}\bm{L}\right)^{-1}\right|\right|_{2}=\max\left\{ \text{spec}\left(\left(\bm{L}^{\mathsf{T}}\bm{L}\right)^{-1}\right)\right\} =\sfrac{1}{\bm{\lambda}_{1}}=\sfrac{1}{\lambda_{\min}\left(\bm{L}^{\mathsf{T}}\bm{L}\right)}$.
Noticing that $\text{spec}\left(\bm{L}^{\mathsf{T}}\bm{L}\right)=\text{spec}\left(\bm{L}\bm{L}^{\mathsf{T}}\right)$,
we recognize $\left|\left|\left(\bm{L}^{\mathsf{T}}\bm{L}\right)^{-1}\right|\right|_{2}=\sfrac{1}{\bm{\lambda}_{\text{min}}\left(\bm{L}\bm{L}^{\mathsf{T}}\right)}\leq\sfrac{1}{p}$.
Therefore 
\begin{equation}
\left|\left|\bm{L}^{+}\right|\right|_{\mathsf{2}}\leq\frac{1}{p}\left|\left|\bm{L}\right|\right|_{\mathsf{F}} \enspace .\label{Eq:L}
\end{equation}
Plugging the results of Eq.~\eqref{Eq:L} into Eq.~\eqref{Eq:S},
we arrive at 
\begin{equation}
\left\|\bm{s}_{\vTaskJ}\Big|_{\hat{\bm{\theta}}_{\vCurrentRound}}\right\|_{2}\leq\sfrac{1}{p}\left\|\bm{L}\Big|_{\hat{\bm{\theta}}_{\vCurrentRound}}\right\|_{\mathsf{F}}\max_{\vTaskK\in\mathcal{I}_{\vCurrentRound-1}}\left\{ \left|\left|\bm{A}_{\vTaskK}^{+}\right|\right|_{2}\left(\left|\left|\bm{b}_{\vTaskK}\right|\right|_{2}+\bm{c}_{\text{max}}\right)\right\} \enspace . \label{Eq:SS}
\end{equation}
Finally, since $\hat{\bm{\theta}}_{\vCurrentRound}$ satisfies
the constraints, we note that $\left\|\bm{L}\Big|_{\hat{\bm{\theta}}_{\vCurrentRound}}\right\|_{\mathsf{F}}^{2}\leq q\times d$.
Consequently, 
\begin{align*}
\Big|\Big|\nabla_{\bm{\theta}}l_{\vTaskJ}\left(\bm{\theta}\right)\Big|_{\hat{\bm{\theta}}_{\vCurrentRound}}\Big|\Big|_{2}^{2} & \leq\Big|\Big|\nabla_{\bm{\alpha}_{\vTaskJ}}l_{\vTaskJ}\left(\bm{\theta}\right)\Big|_{\hat{\bm{\theta}}_{\vCurrentRound}}\Big|\Big|_{2}^{2}\Bigg(q\times d \left(\frac{2d}{p^{2}}\max_{\vTaskK\in\mathcal{I}_{\vCurrentRound-1}}\left\{ \left|\left|\bm{A}_{\vTaskK}^{\dagger}\right|\right|_{2}^{2}\left(\left|\left|\bm{b}_{\vTaskK}\right|\right|_{2}^{2}+\bm{c}_{\text{max}}^{2}\right)\right\} +1\right)\Bigg) \enspace .
\end{align*}
\end{proof}

\begin{mylemma3} The L$_{2}$ norm of the constraint solution at round
$\vCurrentRound-1$, $\|\hat{\bm{\theta}}_{\vCurrentRound}\|_{2}^{2}$
is bounded by 
\[
\|\hat{\bm{\theta}}_{\vCurrentRound}\|_{2}^{2}\leq q\times d\left[1+\left|\mathcal{I}_{\vCurrentRound-1}\right|\frac{1}{p^{2}}\max_{\vTaskK\in\mathcal{I}_{\vCurrentRound-1}}\left\{ \left|\left|\bm{A}_{\vTaskK}^{\dagger}\right|\right|_{2}^{2}\left(\left|\left|\bm{b}_{\vTaskK}\right|\right|_{2}+\bm{c}_{\text{max}}\right)^{2}\right\} \right] \enspace .
\]
with $\left|\mathcal{I}_{\vCurrentRound-1}\right|$ being the cardinality
of $\mathcal{I}_{\vCurrentRound-1}$ representing the number of
different tasks observed so-far. \end{mylemma3} \begin{proof} Noting
that $\hat{\bm{\theta}}_{\vCurrentRound}=\Big[\underbrace{\bm{\theta}_{1},\dots,\bm{\theta}_{dk}}_{\bm{L}\Big|_{\hat{\bm{\theta}}_{\vCurrentRound}}},\underbrace{\bm{\theta}_{dk+1},\dots}_{\bm{s}_{i_{1}}\Big|_{\hat{\theta}_{\vCurrentRound}}},\underbrace{\dots,\dots}_{\bm{s}_{i_{\vCurrentRound-1}}\Big|_{\hat{\theta_{\vCurrentRound}}}},\underbrace{\dots,\bm{\theta}_{dk+kT^{\star}}}_{\text{\ensuremath{\bm{0}}'s: unobserved tasks}}\Big]^{\mathsf{T}}$,
it is easy to see 
\begin{align*}
\|\hat{\bm{\theta}}_{\vCurrentRound}\|_{2}^{2} & \leq\left\|\bm{L}\Big|_{\hat{\theta}_{\vCurrentRound}}\right\|_{\mathsf{F}}^{2}+\left|\mathcal{I}_{\vCurrentRound-1}\right|\max_{\vTaskK\in\mathcal{I}_{\vCurrentRound-1}}\left\{ \left\|\bm{s}_{\vTaskK}\Big|_{\hat{\bm{\theta}}_{\vCurrentRound}}\right\|_{2}^{2}\right\} \\
 & \leq q\times d+\left|\mathcal{I}_{\vCurrentRound-1}\right|\max_{\vTaskK\in\mathcal{I}_{\vCurrentRound-1}}\left[\frac{q\times d}{p^{2}}\left|\left|\bm{A}_{\vTaskK}^{\dagger}\right|\right|_{2}^{2}\left(\left|\left|\bm{b}_{\vTaskK}\right|\right|_{2}+\bm{c}_{\text{max}}\right)^{2}\right]\\
 & \leq q\times d\left[1+\left|\mathcal{I}_{\vCurrentRound-1}\right|\frac{1}{p^{2}}\max_{\vTaskK\in\mathcal{I}_{\vCurrentRound-1}}\left\{ \left|\left|\bm{A}_{\vTaskK}^{\dagger}\right|\right|_{2}^{2}\left(\left|\left|\bm{b}_{\vTaskK}\right|\right|_{2}+\bm{c}_{\text{max}}\right)^{2}\right\} \right].\\
\end{align*}
\end{proof}

\begin{mylemma4} The L$_{2}$ norm of the linearizing term of $l_{\vTaskJ}(\bm{\theta})$
around $\hat{\bm{\theta}}_{\vCurrentRound}$, $\left\|\hat{\bm{f}}_{\vTaskJ}\Big|_{\hat{\bm{\theta}}_{\vCurrentRound}}\right\|_{2}$,
is bounded by 
\begin{align*}
\left\|\hat{\bm{f}}_{\vTaskJ}\Big|_{\hat{\bm{\theta}}_{\vCurrentRound}}\right\|_{2} & \leq\left\|\nabla_{\bm{\theta}}l_{\vTaskJ}(\bm{\theta})\Big|_{\hat{\bm{\theta}}_{\vCurrentRound}}\right\|_{2}\left(1+\|\hat{\bm{\theta}}_{\vCurrentRound}\|_{2}\right)+\left|l_{\vTaskJ}(\bm{\theta})\Big|_{\hat{\bm{\theta}}_{\vCurrentRound}}\right|\leq\bm{\gamma}_{1}(\vCurrentRound)\left(1+\bm{\gamma}_{2}(\vCurrentRound)\right)+\bm{\delta}_{l_{\vTaskJ}},
\end{align*}
with $\bm{\delta}_{l_{\vTaskJ}}$ being the constant upper-bound
on $\left|l_{\vTaskJ}(\bm{\theta})\Big|_{\hat{\bm{\theta}}_{\vCurrentRound}}\right|$,
and 
\begin{align*}
\bm{\gamma}_{1}(\vCurrentRound) & =\frac{1}{n_{\vTaskJ}\sigma_{\vTaskJ}^{2}}\left[\left(u_{\text{max}}+\max_{\vTaskK\in\mathcal{I}_{\vCurrentRound-1}}\left\{ \left|\left|\bm{A}_{\vTaskK}^{+}\right|\right|_{2}\left(\left|\left|\bm{b}_{\vTaskK}\right|\right|_{2}+\bm{c}_{\text{max}}\right)\right\} \bm{\Phi}_{\text{max}}\right)\bm{\Phi}_{\text{max}}\right]\\
 & \hspace{8em}\times\left(\sfrac{d}{p}\sqrt{2q}\sqrt{\max_{\vTaskK\in\mathcal{I}_{\vCurrentRound-1}}\left\{ \|\bm{A}_{\vTaskK}^{\dagger}\|_{2}^{2}\left(\|\bm{b}_{\vTaskK}\|_{2}^{2}+\bm{c}_{\text{max}}^{2}\right)\right\} }+\sqrt{qd}\right) \enspace .\\
\bm{\gamma}_{2}(\vCurrentRound) & \leq\sqrt{q\times d}+\sqrt{\left|\mathcal{I}_{\vCurrentRound-1}\right|}\sqrt{\left[1+\frac{1}{p^{2}}\max_{\vTaskK\in\mathcal{I}_{\vCurrentRound-1}}\left\{ \left|\left|\bm{A}_{\vTaskK}^{\dagger}\right|\right|_{2}^{2}\left(\left|\left|\bm{b}_{\vTaskK}\right|\right|_{2}+\bm{c}_{\text{max}}\right)^{2}\right\} \right]} \enspace .
\end{align*}
\end{mylemma4}

\begin{proof} We have previously shown that $\Big|\Big|\hat{\bm{f}}_{\vTaskJ}\Big|_{\hat{\bm{\theta}}_{\vCurrentRound}}\Big|\Big|_{2}\leq\Big|\Big|\nabla_{\bm{\theta}}l_{\vTaskJ}\left(\bm{\theta}\right)\Big|_{\hat{\bm{\theta}}_{\vCurrentRound}}\Big|\Big|_{2}+\Big|l_{\vTaskJ}\left(\hat{\theta}_{\vCurrentRound}\right)\Big|+\Big|\Big|\nabla_{\bm{\theta}}l_{\vTaskJ}(\bm{\theta})\Big|_{\hat{\bm{\theta}}_{\vCurrentRound}}\Big|\Big|_{2}\times\Big|\Big|\hat{\bm{\theta}}_{\vCurrentRound}\Big|\Big|_{2}$.
Using the previously derived lemmas we can upper-bound $\Big|\Big|\hat{\bm{f}}_{\vTaskJ}\Big|_{\hat{\bm{\theta}}_{\vCurrentRound}}\Big|\Big|_{2}$
as follows 
\begin{align*}
\Big|\Big|\nabla_{\bm{\theta}}l_{\vTaskJ}\left(\bm{\theta}\right)\Big|_{\hat{\bm{\theta}}_{\vCurrentRound}}\Big|\Big|_{2}^{2} & \leq\Big|\Big|\nabla_{\bm{\alpha}_{\vTaskJ}}l_{\vTaskJ}\left(\bm{\theta}\right)\Big|_{\hat{\bm{\theta}}_{\vCurrentRound}}\Big|\Big|_{2}^{2}\Bigg(q\times d\left(\frac{2d}{p^{2}}\max_{\vTaskK\in\mathcal{I}_{\vCurrentRound-1}}\left\{ \left|\left|\bm{A}_{\vTaskK}^{\dagger}\right|\right|_{2}^{2}\left(\left|\left|\bm{b}_{\vTaskK}\right|\right|_{2}^{2}+\bm{c}_{\text{max}}^{2}\right)\right\} +1\right)\Bigg)\\
\Big|\Big|\nabla_{\bm{\theta}}l_{\vTaskJ}\left(\bm{\theta}\right)\Big|_{\hat{\bm{\theta}}_{\vCurrentRound}}\Big|\Big|_{2} & \leq\Big|\Big|\nabla_{\bm{\alpha}_{\vTaskJ}}l_{\vTaskJ}\left(\bm{\theta}\right)\Big|_{\hat{\bm{\theta}}_{\vCurrentRound}}\Big|\Big|_{2}\left(\sfrac{d}{p}\sqrt{2q}\sqrt{\max_{\vTaskK\in\mathcal{I}_{\vCurrentRound-1}}\left\{ \|\bm{A}_{\vTaskK}^{\dagger}\|_{2}^{2}\left(\|\bm{b}_{\vTaskK}\|_{2}^{2}+\bm{c}_{\text{max}}^{2}\right)\right\} }+\sqrt{qd}\right)\\
 & \leq\frac{1}{n_{\vTaskJ}\sigma_{\vTaskJ}^{2}}\left[\left(u_{\text{max}}+\max_{\vTaskK\in\mathcal{I}_{\vCurrentRound-1}}\left\{ \left|\left|\bm{A}_{\vTaskK}^{+}\right|\right|_{2}\left(\left|\left|\bm{b}_{\vTaskK}\right|\right|_{2}+\bm{c}_{\text{max}}\right)\right\} \bm{\Phi}_{\text{max}}\right)\bm{\Phi}_{\text{max}}\right]\\
 & \hspace{8em}\times\left(\sfrac{d}{p}\sqrt{2q}\sqrt{\max_{\vTaskK\in\mathcal{I}_{\vCurrentRound-1}}\left\{ \|\bm{A}_{\vTaskK}^{\dagger}\|_{2}^{2}\left(\|\bm{b}_{\vTaskK}\|_{2}^{2}+\bm{c}_{\text{max}}^{2}\right)\right\} }+\sqrt{qd}\right) \enspace .
\end{align*}
Further, 
\begin{align*}
\left|\left|\hat{\bm{\theta}}_{\vCurrentRound}\right|\right|_{2}^{2} & \leq q\times d+\left|\mathcal{I}_{\vCurrentRound-1}\right|\max_{\vTaskK\in\mathcal{I}_{\vCurrentRound-1}}\left[1+\frac{1}{p^{2}}\max_{\vTaskK\in\mathcal{I}_{\vCurrentRound-1}}\left\{ \left|\left|\bm{A}_{\vTaskK}^{\dagger}\right|\right|_{2}^{2}\left(\left|\left|\bm{b}_{\vTaskK}\right|\right|_{2}+\bm{c}_{\text{max}}\right)^{2}\right\} \right]\\
\implies \left|\left|\hat{\bm{\theta}}_{\vCurrentRound}\right|\right|_{2} & \leq\sqrt{q\times d}+\sqrt{\left|\mathcal{I}_{\vCurrentRound-1}\right|}\sqrt{\left[1+\frac{1}{p^{2}}\max_{\vTaskK\in\mathcal{I}_{\vCurrentRound-1}}\left\{ \left|\left|\bm{A}_{\vTaskK}^{\dagger}\right|\right|_{2}^{2}\left(\left|\left|\bm{b}_{\vTaskK}\right|\right|_{2}+\bm{c}_{\text{max}}\right)^{2}\right\} \right]} \enspace .
\end{align*}
Therefore 
\begin{align}
\left|\left|\hat{\bm{f}}_{\vTaskJ}\Big|_{\hat{\bm{\theta}}_{\vCurrentRound}}\right|\right| & \leq\left|\left|\nabla_{\bm{\theta}}l_{\vTaskJ}(\bm{\theta})\Big|_{\hat{\bm{\theta}}_{\vCurrentRound}}\right|\right|_{2}\left(1+\left|\left|\hat{\bm{\theta}}_{\vCurrentRound}\right|\right|_{2}\right)+\left|l_{\vTaskJ}(\bm{\theta})\Big|_{\hat{\bm{\theta}}_{\vCurrentRound}}\right|\label{Eq:f}\\
 & \leq\bm{\gamma}_{1}(\vCurrentRound)\left(1+\bm{\gamma}_{2}(\vCurrentRound)\right)+\bm{\delta}_{l_{\vTaskJ}},\nonumber 
\end{align}
with $\bm{\delta}_{l_{\vTaskJ}}$ being the constant upper-bound
on $\left|l_{\vTaskJ}(\bm{\theta})\Big|_{\hat{\bm{\theta}}_{\vCurrentRound}}\right|$,
and 
\begin{align*}
\bm{\gamma}_{1}(\vCurrentRound) & =\frac{1}{n_{\vTaskJ}\sigma_{\vTaskJ}^{2}}\left[\left(u_{\text{max}}+\max_{\vTaskK\in\mathcal{I}_{\vCurrentRound-1}}\left\{ \left|\left|\bm{A}_{\vTaskK}^{+}\right|\right|_{2}\left(\left|\left|\bm{b}_{\vTaskK}\right|\right|_{2}+\bm{c}_{\text{max}}\right)\right\} \bm{\Phi}_{\text{max}}\right)\bm{\Phi}_{\text{max}}\right]\\
 & \hspace{8em}\times\left(\sfrac{d}{p}\sqrt{2q}\sqrt{\max_{\vTaskK\in\mathcal{I}_{\vCurrentRound-1}}\left\{ \|\bm{A}_{\vTaskK}^{\dagger}\|_{2}^{2}\left(\|\bm{b}_{\vTaskK}\|_{2}^{2}+\bm{c}_{\text{max}}^{2}\right)\right\} }+\sqrt{qd}\right).\\
\bm{\gamma}_{2}(\vCurrentRound) & \leq\sqrt{q\times d}+\sqrt{\left|\mathcal{I}_{\vCurrentRound-1}\right|}\sqrt{\left[1+\frac{1}{p^{2}}\max_{\vTaskK\in\mathcal{I}_{\vCurrentRound-1}}\left\{ \left|\left|\bm{A}_{\vTaskK}^{\dagger}\right|\right|_{2}^{2}\left(\left|\left|\bm{b}_{\vTaskK}\right|\right|_{2}+\bm{c}_{\text{max}}\right)^{2}\right\} \right]} \enspace .
\end{align*}
\end{proof}

\begin{mytheo}[Sublinear Regret; restated from the main paper] After $\vNumTotalRounds$
rounds and choosing $\eta_{\vTaskOne}=\dots=\eta_{\vTaskJ}=\eta=\frac{1}{\sqrt{\vNumTotalRounds}}$,
$\bm{L}\Big|_{\hat{\bm{\theta}}_{1}}=\text{diag}_{\bm{k}}(\zeta)$,
with $\text{diag}_{\bm{k}}(\cdot)$ being a diagonal matrix among
the $\bm{k}$ columns of $\bm{L}$, $p\leq\zeta^{2}\leq q$, and $\bm{S}\Big|_{\hat{\bm{\theta}}_{1}}=\bm{0}_{k\times \vNumTotalTasks}$,
for any $\bm{u}\in\mathcal{K}$ our algorithm exhibits a sublinear regret of the form 
\[
\sum_{j=1}^{\vNumTotalRounds}l_{\vTaskJ}\left(\hat{\bm{\theta}}_{\vCurrentRound}\right)-l_{\vTaskJ}(\bm{u})=\mathcal{O}\left(\sqrt{\vNumTotalRounds}\right) \enspace .
\]
\end{mytheo}

\begin{proof} Given the ingredients of the previous section, next
we derive the sublinear regret results which finalize the statement
of the theorem. First, it is easy to see that 
\begin{equation*}
\nabla_{\bm{\theta}}\bm{\Omega}_{0}\left(\tilde{\bm{\theta}}_{j}\right)-\nabla_{\bm{\theta}}\bm{\Omega}_{0}\left(\tilde{\bm{\theta}}_{j+1}\right)=\eta_{\vTaskJ}\hat{\bm{f}}_{\vTaskJ}\Big|_{\hat{\bm{\theta}}_{j}} \enspace .
\end{equation*}
Further, from strong convexity of the regularizer we obtain: 
\[
\bm{\Omega}_{0}\left(\hat{\bm{\theta}}_{j}\right)\geq\bm{\Omega}_{0}\left(\hat{\bm{\theta}}_{j+1}\right)+\left\langle \nabla_{\bm{\theta}}\bm{\Omega}_{0}\left(\hat{\bm{\theta}}_{j+1}\right),\hat{\bm{\theta}}_{j}-\hat{\bm{\theta}}_{j+1}\right\rangle +\frac{1}{2}\left|\left|\hat{\bm{\theta}}_{j}-\hat{\bm{\theta}}_{j+1}\right|\right|_{2}^{2} \enspace .
\]
It can be seen that 
\begin{equation*}
\left\|\hat{\bm{\theta}}_{j}-\hat{\bm{\theta}}_{j+1}\right\|_{2}\leq\eta_{\vTaskJ}\left\|\hat{\bm{f}}_{\vTaskJ}\Big|_{\hat{\bm{\theta}}_{j}}\right\|_{2} \enspace .
\end{equation*}
Finally, for any $\bm{u}\in\mathcal{K}$, we have: 
\begin{equation*}
\sum_{j=1}^{\vCurrentRound}\eta_{\vTaskJ}\left(l_{\vTaskJ}\left(\hat{\bm{\theta}}_{j}\right)-l_{\vTaskJ}(\bm{u})\right)\leq\sum_{j=1}^{\vCurrentRound}\left[\eta_{\vTaskJ}\left(\left\|\hat{\bm{f}}_{\vTaskJ}\Big|_{\hat{\bm{\theta}}_{j}}\right\|_{2}\right)^{2}\right]+\bm{\Omega}_{0}(\bm{u})-\bm{\Omega}_{0}(\hat{\bm{\theta}}_{1}) \enspace .
\end{equation*}
Assuming $\eta_{\vTaskOne}=\dots=\eta_{\vTaskJ}=\eta$, we can derive 
\begin{equation*}
\sum_{j=1}^{\vCurrentRound}\left(l_{\vTaskJ}\left(\hat{\bm{\theta}}_{j}\right)-l_{\vTaskJ}(\bm{u})\right)\leq\eta\sum_{j=1}^{\vCurrentRound}\left(\left\|\hat{\bm{f}}_{\vTaskJ}\Big|_{\hat{\bm{\theta}}_{j}}\right\|_{2}\right)^{2}+\sfrac{1}{\eta}\left(\bm{\Omega}_{0}(\bm{u})-\bm{\Omega}_{0}(\hat{\bm{\theta}}_{1})\right) \enspace .
\end{equation*}

The following lemma finalizes the statement of the theorem: 
\begin{mylemma5}
After T rounds and for $\eta_{\vTaskOne}=\dots=\eta_{\vTaskJ}=\eta=\frac{1}{\sqrt{\vNumTotalRounds}}$,
our algorithm exhibits, for any $\bm{u}\in\mathcal{K}$, a sublinear
regret of the form 
\begin{equation*}
\sum_{j=1}^{\vNumTotalRounds}l_{\vTaskJ}(\hat{\bm{\theta}}_{j})-l_{\vTaskJ}(\bm{u})\leq\mathcal{O}\left(\sqrt{\vNumTotalRounds}\right) \enspace .
\end{equation*}
\end{mylemma5} \begin{proof} It is then easy to see 
\begin{align*}
\left\|\hat{\bm{f}}_{\vTaskJ}\Big|_{\hat{\bm{\theta}}_{\vCurrentRound}}\right\|_{2}^{2} & \leq\bm{\gamma}_{3}(\vNumTotalRounds)+4\bm{\gamma}_{1}^{2}(\vNumTotalRounds)\bm{\gamma}_{2}^{2}(\vNumTotalRounds)\ \ \ \ \ \text{with}\ \ \ \ \ \bm{\gamma}_{3}(\vNumTotalRounds)=4\bm{\gamma}_{1}^{2}(\vNumTotalRounds)+2\max_{\vTaskJ\in\mathcal{I}_{\vNumTotalRounds-1}}\bm{\delta}_{\vTaskJ}^{2}\\
 & \leq\bm{\gamma}_{3}(\vNumTotalRounds)+8\frac{d}{p^{2}}\bm{\gamma}_{1}^{2}(\vNumTotalRounds)qd +8\frac{d}{p^{2}}\bm{\gamma}_{1}^{2}(\vNumTotalRounds)qd\left|\mathcal{I}_{\vNumTotalRounds-1}\right|\max_{\vTaskK\in\mathcal{I}_{\vNumTotalRounds-1}}\left\{ \|\bm{A}_{\vTaskK}^{\dagger}\|_{2}\left(\|\bm{b}_{\vTaskK}\|_{2}+\bm{c}_{\text{max}}\right)^{2}\right\} \enspace . \\
\end{align*}
Since $|\mathcal{I}_{\vNumTotalRounds-1}|\leq \vNumTotalTasks$ with $\vNumTotalTasks$
being the total number of tasks available, then we can write 
\[
\left\|\hat{\bm{f}}_{\vTaskJ}\Big|_{\hat{\theta}_{\vCurrentRound}}\right\|_{2}^{2}\leq\bm{\gamma}_{5}(\vNumTotalRounds)\vNumTotalTasks \enspace ,
\]
with $\bm{\gamma}_{5}=8\sfrac{d}{p^{2}}q\bm{\gamma}_{1}^{2}(\vNumTotalRounds)\max_{\vTaskK\in\mathcal{I}_{\vNumTotalRounds-1}}\left\{ \|\bm{A}_{\vTaskK}^{\dagger}\|_{2}^{2}\left(\|\bm{b}_{\vTaskK}\|_{2}+\bm{c}_{\text{max}}\right)^{2}\right\} $.
Further, it is easy to see that $\bm{\Omega}_{0}(\bm{u})\leq qd+\bm{\gamma}_{5}(\vNumTotalRounds)\vNumTotalTasks$
with $\bm{\gamma}_{5}(\vNumTotalRounds)$ being a constant, which
leads to 
\[
\sum_{j=1}^{\vCurrentRound}\left(l_{\vTaskJ}\left(\hat{\bm{\theta}}_{j}\right)-l_{\vTaskJ}(\bm{u})\right)\leq\eta\sum_{j=1}^{\vCurrentRound}\bm{\gamma}_{5}(\vNumTotalRounds)\vNumTotalTasks+\sfrac{1}{\eta}\left(qd+\bm{\gamma}_{5}(\vNumTotalRounds)\vNumTotalTasks-\bm{\Omega}_{0}(\hat{\bm{\theta}}_{1})\right) \enspace .
\]
\textbf{Initializing $\bm{L}$ and $\bm{S}$:} We initialize $\bm{L}\Big|_{\hat{\bm{\theta}}_{1}}=\text{diag}_{\bm{k}}(\zeta)$,
with $p\leq\zeta^{2}\leq q$ and $\bm{S}\Big|_{\hat{\bm{\theta}}_{1}}=\bm{0}_{k\times \vNumTotalTasks}$
ensures the invertability of $\bm{L}$ and that the constraints are
met. This leads us to 
\[
\sum_{j=1}^{\vCurrentRound}\left(l_{\vTaskJ}\left(\hat{\bm{\theta}}_{j}\right)-l_{\vTaskJ}(\bm{u})\right)\leq\eta\sum_{j=1}^{\vCurrentRound}\bm{\gamma}_{5}(\vNumTotalRounds)\vNumTotalTasks+\sfrac{1}{\eta}\left(qd+\bm{\gamma}_{5}(\vNumTotalRounds)\vNumTotalTasks-\mu_{2}k\zeta\right) \enspace .
\]
Choosing $\eta_{\vTaskOne}=\dots=\eta_{\vTaskJ}=\eta=\sfrac{1}{\sqrt{\vNumTotalRounds}}$,
we acquire sublinear regret, finalizing the statement of the theorem:
\begin{align*}
\sum_{j=1}^{\vCurrentRound}\left(l_{\vTaskJ}\left(\hat{\bm{\theta}}_{j}\right)-l_{\vTaskJ}(\bm{u})\right) & \leq\sfrac{1}{\sqrt{\vNumTotalRounds}}\bm{\gamma}_{5}(\vNumTotalRounds)\vNumTotalTasks\vNumTotalRounds+\sqrt{\vNumTotalRounds}\left(qd+\bm{\gamma}_{5}(\vNumTotalRounds)\vNumTotalTasks-\mu_{2}k\zeta\right)\\
 & \leq\sqrt{\vNumTotalRounds}\left(\bm{\gamma}_{5}(\vNumTotalRounds)\vNumTotalTasks+qd\bm{\gamma}_{5}(\vNumTotalRounds)\vNumTotalTasks-\mu_{2}k\zeta\right)\leq\mathcal{O}\left(\sqrt{\vNumTotalRounds}\right) \enspace ,
\end{align*}
with $\bm{\gamma}_{5}(\vNumTotalRounds)$ being a constant. \end{proof}
\end{proof}

\end{document}